\documentclass[12pt]{article}
\usepackage[margin=1.2in]{geometry}

\usepackage[utf8]{inputenc}
\usepackage[english]{babel}
\usepackage[dvipsnames]{xcolor}
\usepackage{amsmath,latexsym,amssymb,amsthm,mathtools,dsfont}
\usepackage[numbers,square,sort&compress]{natbib}
\usepackage{caption}
\usepackage{subcaption}
\usepackage{blkarray}
\usepackage{cancel}
\usepackage{enumitem}
\usepackage{euscript}
\usepackage{float}
\usepackage{graphicx}
\usepackage{mathrsfs}
\usepackage{microtype}
\usepackage{ragged2e}
\usepackage{thmtools}
\usepackage{tikz}
\usepackage[Symbolsmallscale]{upgreek}
\usepackage{algorithm}
\usepackage{algpseudocode}
\usepackage{amsopn}
\usepackage[colorlinks=true,linkcolor=blue,citecolor=blue,urlcolor=black,hypertexnames=false]{hyperref}

\definecolor{darkgreen}{rgb}{0,0.5,0} %
\definecolor{darkblue}{rgb}{1,0,0}

\newcommand{\acknowledgments}[1]{\section*{Acknowledgments}#1}

\input{mathsymbols}

\DeclareRobustCommand{\modecollapsesectionref}{\Cref{sec:mode-collapse}}
\DeclareRobustCommand{\multimodecollapsesectionref}{\Cref{sec:multimodal}}
\DeclareRobustCommand{\rovisectionref}{\Cref{sec:rotation-vi}}
\renewenvironment{thebibliography}[1]{%
\begin{oldthebibliography}{#1}%
\setlength{\baselineskip}{.85em}
\linespread{.9}
\small
\setlength{\parskip}{.07ex}%
\setlength{\itemsep}{.07em}%
}%
{%
\end{oldthebibliography}%
}

\title{Mode Collapse of Mean-Field Variational Inference}
\date{}
\author{  
 Shunan Sheng%
  \thanks{
  Columbia University, Department of Statistics, ss6574@columbia.edu.}  \and
  Bohan Wu%
  \thanks{ Columbia University, Department of Statistics, bw2766@columbia.edu.} 
  \and
  Alberto Gonz{\'a}lez-Sanz%
  \thanks{Department of Statistics, Columbia University, ag4855@columbia.edu} }  
\hypersetup{pdftitle={Mode Collapse of Mean-Field Variational Inference},pdfauthor={Shunan Sheng, Bohan Wu, Alberto Gonz\'alez-Sanz}}

\begin{document}
\maketitle

\begin{abstract}

Mean-field variational inference (MFVI) is a widely used method for approximating high-dimensional probability distributions by product measures. It has been empirically observed that MFVI optimizers often suffer from \emph{mode collapse}. For example, when the target measure $\pi$ is a mixture $\pi = w P_0 + (1 - w) P_1$, the MFVI optimizer tends to place most of its mass near a single component of the mixture. This work provides the first \emph{theoretical} explanation of mode collapse in MFVI. Starting from two-component mixtures, we introduce the notion of \emph{$\varepsilon$-separateness} to quantify the separation of the mixture components. We derive explicit bounds on the mass that any MFVI optimizer assigns to each component, when $P_0$ and $P_1$ are $\varepsilon$-separated for sufficiently small $\varepsilon$. We further extend this analysis to mixture target distributions with multiple components. Our results reveal that the occurrence of mode collapse crucially depends on the {relative position of the components}. To address this issue, we propose \emph{rotational variational inference (RoVI)}, which augments MFVI with a rotation matrix. The numerical studies support our theoretical findings and illustrate improved mode coverage with RoVI in Bayesian mixture posteriors, including a real-data example.
\end{abstract}

\section{Introduction}
Given a probability measure $\pi$ on $\R^d$ with potential $V: \R^d \mapsto \R\cup \{+\infty\}$, i.e.,
\begin{equation*}
\pi(\rd x)  = Z^{-1}e^{-V(x)}\dd x 
\end{equation*}
where $Z$ denotes the normalizing constant.  The problem of \emph{sampling} is to draw Monte Carlo samples from $\pi$ where the normalizing constant is unknown. 

Variational inference (VI) is a widely used approximation technique for sampling from high-dimensional probability distributions \citep{Blei2017}. It has been applied in numerous large-scale Bayesian inference problems, including topic modeling \citep{Blei2003}, deep generative modeling \citep{KingmaWelling2019,Lopez2018}, robust Bayes analysis \citep{Wang2018robustBayesianmodeling}, marketing science \citep{Braun2010}, and genome sequence modeling \citep{Carbonetto2012,Lopez2018,Wang2020EBVI-VS}. In these applications, obtaining exact posterior samples is often computationally infeasible, whereas a well-calibrated variational approximation is sufficient for practical use \citep{Blei2017}.

The mean-field variational inference (MFVI) problem is to find the best approximation of $\pi$ within the class of product measures $ \cP(\R)^{\otimes d}$ under the Kullback--Leibler~(KL) divergence, 
\begin{equation}\label{eq: MFVI}
   \mu^* \in \argmin_{\mu \in \cP(\R)^{\otimes d}} \kl{\mu}{\pi}. \tag{MFVI}
\end{equation}

When the target distribution $\pi$ is log-concave, MFVI enjoys several computational and theoretical guarantees~\citep{arnese2024convergence, jiang2025algorithms, Wang2019,Lacker2024,lavenant2024convergence,wu2024extending}. In particular, \cite{Sheng2025Stability} establish quantitative stability results for MFVI optimizers. However, the stability of MFVI in the non–log-concave setting remains an open question, with both empirical and theoretical evidence suggesting potential instability \citep{Ghorbani2019,soletskyi2024theoretical,blessing2024beyond}. A well-documented observation is the mode-collapsing phenomenon, often caused by the label-switching problem in statistics \citep{Stephens2000}, where the solution of variational inference problem ignores all but one mode of a mixture distribution \citep{Wang2006,Blei2017,pati2018statistical}. Our preliminary experiment in \Cref{fig:mode-collapsing-Gaussian} confirms that MFVI suffers from mode collapse: while the target measure has two modes, the algorithm only captures one of the modes depending on the initialization, suggesting that the performance of MFVI could be unstable. Related phenomena arise in the sampling literature under the rubric of \emph{metastability}, where the MCMC algorithm (e.g., Langevin Monte Carlo) can require exponential time to jump between different modes \citep{Bovier2002, Bovier2004, Bovier2005, balasubramanian2022towards}.

To analyze mode collapse, we introduce a notion of separateness for two-component mixture distributions called \emph{$\vae$-separated} (see \cref{defn:epsilon-separated}). This notion is formulated via orthogonal half-spaces, which adapts to the geometry of the MFVI problem. The proposed notion parallels the $c$-separateness introduced in~\cite{Dasgupta1999}, which quantifies overlap between Gaussian components.  Our work shows that MFVI is particularly vulnerable to mode collapse when the components are $\vae$-separated with sufficiently small $\vae$. 
The focus on mixture distributions is canonical in the study of non-log-concave sampling, as one cannot expect a unified theory for approximating non–log-concave targets \citep{Chewi2023}. Understanding how MFVI behaves in this setting could provide insights into the behavior of variational inference for complex multi-modal targets. 
Along these lines, most prior work on learning mixture distributions has focused on Gaussian mixtures~\cite{Wu2020GM, Yan2023, Jin2016}. For example, \cite{Dasgupta1999,Arora2005} analyze the problem of learning mixtures of separated Gaussian components in polynomial time, and \cite{Ge2015} establishes the sample and iterative complexity for learning Gaussian mixture distributions in general dimensions. 

In this work, we present a rigorous theoretical study of mode collapse in MFVI for mixture targets. We first consider a two-component mixture $\pi$ and show in \Cref{thm:MFVI-concentration} that every MFVI optimizer concentrates around a single component when the components are approximately separated along two coordinate directions. Under exact separation, we further establish instability of MFVI with respect to the mixing weight. We then extend the concentration result to mixtures with multiple components in~\cref{thm:multimodal}. To mitigate mode collapse in the two-component setting, we propose \emph{rotational variational inference (RoVI)}, which augments the family of product measures with rotations. For two-component Gaussian mixtures with a common covariance matrix, we derive an upper bound on the minimum KL divergence over this enlarged variational family. We validate the effectiveness of RoVI via empirical experiments. A concurrent work \cite{chen2025rotated} develops iterative Gaussianization through score-informed rotations and rotated-MFVI subproblems. Our approach jointly optimizes a single rotation and a product approximation.

\begin{figure}
    \centering
    \includegraphics[width=\linewidth]{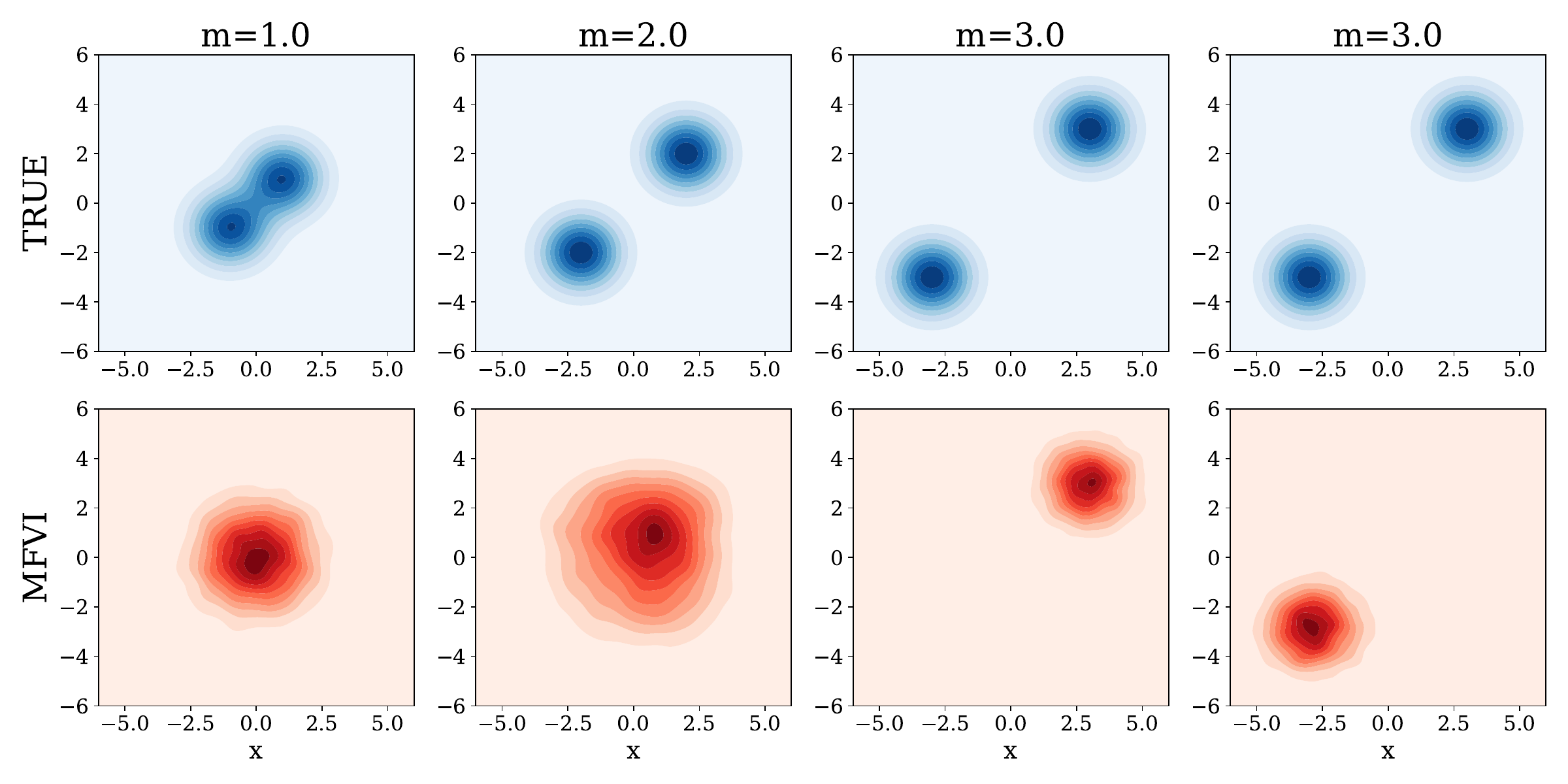}
    \caption{
    $\pi = \frac{1}{2}\,\cN\left([-m,-m]^\top, I_2
    \right) + \frac{1}{2}\,\cN\left([m,m]^\top, I_2
    \right)$. The above figure shows the contour plots of $\pi$ and a fitted MFVI approximation when $m \in \{1,2,3\}$. We observe the mode-collapse is prominent when $m=3$. Moreover, when $m=3$, the fitted approximation can switch from one mode to the other depending on the initialization.  }
    \label{fig:mode-collapsing-Gaussian}
\end{figure}

\textbf{Organization. } The rest of the paper is organized as follows. \Cref{sec:mode-collapse} presents theoretical results explaining mode collapse in MFVI when the target is a two-component mixture. \Cref{sec:multimodal} then extends the concentration result to mixtures with multiple components. In \Cref{sec:rotation-vi}, we propose the rotational variational inference algorithm (RoVI) and present numerical studies. \Cref{sect-discussion} concludes the paper with a discussion on future directions. The proofs of our main results are included in the Appendices.
\subsection{Notations}
For any $d \in \mathbb{Z}_+ = \{1, 2, \dots\}$, we denote $[d] = \{1, 2, \dots, d\}$ and let $I_d$ be the identity matrix of size $d$. Given a set $\cX\subset \R^d$, we denote by $\mathcal{P}(\cX)$ the set of probability distributions on $\cX$. Let $\cP_{p, \rm ac}(\R^d)$ be the set of all probability measures that are absolutely continuous with respect to the Lebesgue measure with finite $p$-th moments. Let $\cP(\R)^{\otimes d} \subset \cP(\R^d)$ be the set of all product probability measures. The Kullback--Leibler (KL) divergence between $\mu, \pi \in \cP(\R^d)$ is defined as $\kl{\mu}{\pi} = \int \log\left( \frac{\dd\mu}{\dd \pi} \right) \dd\mu$ if $\mu$ is absolutely continuous with respect to $\pi$ (denoted $\mu \ll \pi$) and $\kl{\mu}{\pi} = \infty$ otherwise. Here, $\frac{\dd\mu}{\dd \pi}$ denotes the Radon--Nikodym derivative of $\mu$ with respect to $\pi$. We denote by $\cL_d$ the Lebesgue measure on $\R^d$. When $\mu\ll\cL_d$, we use $\mu$ to denote both the measure and its density with respect to $\cL_d$. The (negative) differential entropy is then defined as $H(\mu) = \int \log(\mu(x)) \mu(\rd x)$ if $\mu\ll\cL_d$ and $\infty$ otherwise. 
We use $\|\cdot\|_2$ to denote the Euclidean norm. 
For any $m \in \R^d$ and any $d \times d$ positive definite matrix $\Sigma$, we denote by $\cN(m, \Sigma)$ the multivariate normal distribution with mean vector $m$ and covariance matrix $\Sigma$. For quantities $a,b$, we write $a = O(b)$ if there exists a constant $C > 0$ (which may depend on other parameters depending on context) such that $a \leq C b$.

\section{Mode Collapse}\label{sec:mode-collapse}

In this section, we study the mode-collapsing phenomenon of MFVI when the target measure $\pi$ is a mixture distribution with two components $P_0,\,P_1 \in \cP(\R^d)$:
\begin{equation*}\label{Model-fair}
    \pi = w P_0 + (1-w ) P_1,\quad w \in (0,1).
\end{equation*}
As shown in~\cref{fig:mode-collapsing-Gaussian}, the MFVI optimizer collapses to a single component when $P_0,P_1$ are sufficiently separated. To quantify this separation, we introduce the notion of \emph{$\vae$-separatedness}.

\begin{definition}\label{defn:epsilon-separated}
 Given $\vae\in [0,1)$, we say that two probability measures $P_0$ and $P_1$ are {\it $\vae$-separated} if there exist ${j,k}\in \{1, \dots, d\}$ with $j\neq k$,  $b_j, b_k\in \R$ and $s_j, s_k\in \{-1, 1\}$ such that 
\[
P_0(H_j^{-}\cap H_k^{-}) \geq 1-\vae \quad {\rm and}\quad P_1(H_j^{+}\cap H_k^{+}) \geq 1-\vae,
\] 
where
\[
H_i^-= \{ x: s_i x_i <  b_i \}\quad \text{and}\quad  H_i^+=\{ x: s_i x_i >  b_i \}.
\]
We call $H_j^\pm, H_k^\pm$ the $\vae$-separating half-spaces. When $\vae =0$, we say that $P_0$ and $P_1$ are {\it well-separated}. In particular, well-separateness implies
\[
{\rm supp}(P_0) \subseteq \overline{H_j^{-}\cap H_k^{-}}\quad {\rm and}\quad {\rm supp}(P_1) \subseteq \overline{ H_j^{+}\cap H_k^{+}}.
\]
\end{definition}

\begin{figure}[h!]
    \centering
\tikzset{every picture/.style={line width=0.75pt}} %
\resizebox{\columnwidth}{!}{%
\begin{tikzpicture}[scale=0.75,  %
                    x=0.75pt,y=0.75pt,yscale=-1,xscale=1]
\draw    (628,264.69) -- (73.67,264.33) ;
\draw [shift={(631,264.69)}, rotate = 180.04] [fill={rgb, 255:red, 0; green, 0; blue, 0 }  ][line width=0.08]  [draw opacity=0] (8.93,-4.29) -- (0,0) -- (8.93,4.29) -- cycle    ;
\draw    (146.36,8) -- (148.33,267) ;
\draw [shift={(146.33,5)}, rotate = 89.56] [fill={rgb, 255:red, 0; green, 0; blue, 0 }  ][line width=0.08]  [draw opacity=0] (8.93,-4.29) -- (0,0) -- (8.93,4.29) -- cycle    ;
\draw [fill={rgb, 255:red, 74; green, 144; blue, 226 }  ,fill opacity=0.67 ][line width=1.5] [line join = round][line cap = round]   (216.33,159) .. controls (216.33,161.54) and (214.2,163.61) .. (213.33,166) .. controls (208.33,172) and (207.55,181.08) .. (206.33,189) .. controls (203.52,207.28) and (207.28,218.94) .. (219.33,232) .. controls (239.25,253.57) and (252.47,247.99) .. (280.33,245) .. controls (288.32,244.14) and (296.08,245.99) .. (303.33,242) .. controls (308.44,239.19) and (313.16,235.7) .. (318.33,233) .. controls (321.52,231.34) and (327.2,232.41) .. (328.33,229) .. controls (331.53,219.41) and (303.02,213.07) .. (297.33,211) .. controls (293.41,209.57) and (291.15,208.1) .. (290.33,204) .. controls (288.67,195.71) and (286.66,192.32) .. (280.33,186) .. controls (272.03,177.69) and (251.76,180.49) .. (241.33,179) .. controls (233.72,177.91) and (228.65,167.66) .. (223.33,165) .. controls (222.65,164.66) and (216.33,159.4) .. (216.33,160) .. controls (217.53,163.4) and (220.33,158) .. (217.33,162) ;
\draw [fill={rgb, 255:red, 208; green, 2; blue, 27 }  ,fill opacity=0.45 ][line width=1.5] [line join = round][line cap = round]   (395.33,33) .. controls (395.33,35.54) and (393.2,37.61) .. (392.33,40) .. controls (387.33,46) and (386.55,55.08) .. (385.33,63) .. controls (382.52,81.28) and (410.28,74.44) .. (422.33,87.5) .. controls (442.25,109.07) and (431.47,121.99) .. (459.33,119) .. controls (467.32,118.14) and (475.08,119.99) .. (482.33,116) .. controls (487.44,113.19) and (492.16,109.7) .. (497.33,107) .. controls (500.52,105.34) and (528.33,119) .. (534.33,114.5) .. controls (537.53,104.91) and (600.33,87) .. (553.33,82.5) .. controls (522.33,78.5) and (530.33,92.5) .. (543.33,51.5) .. controls (580.33,31) and (477.66,48.82) .. (471.33,42.5) .. controls (463.03,34.19) and (433.76,28.99) .. (423.33,27.5) .. controls (415.72,26.41) and (407.65,41.66) .. (402.33,39) .. controls (401.33,29.5) and (395.33,33.4) .. (395.33,34) .. controls (396.53,37.4) and (399.33,32) .. (396.33,36) ;
\draw [line width=1.5]  [dash pattern={on 1.69pt off 2.76pt}]  (626.33,143.5) -- (75.33,142.5) ;
\draw [line width=1.5]  [dash pattern={on 1.69pt off 2.76pt}]  (364.33,6) -- (367,307.35) ;

\draw (239.33,201.07) node [anchor=north west][inner sep=0.75pt]  [font=\Large]  {$P_{0}$};
\draw (452.67,54.4) node [anchor=north west][inner sep=0.75pt]  [font=\Large]  {$P_{1}$};
\draw (92.67,106.4) node [anchor=north west][inner sep=0.75pt]  [font=\Large]  {$H_{1}^+$};
\draw (92.67,146.4) node [anchor=north west][inner sep=0.75pt]  [font=\Large]  {$H_{1}^-$};
\draw (330,274.4) node [anchor=north west][inner sep=0.75pt]  [font=\Large]  {$H_{0}^-$};
\draw (372.67,274.4) node [anchor=north west][inner sep=0.75pt]  [font=\Large]  {$H_{0}^+$};
\end{tikzpicture}
}
    \caption{Example of well-separated  $  P_0$ and $P_1$. The blue and red regions represent the supports of $P_0$ and $P_1$, respectively. }
    \label{fig:well-separated}
\end{figure}

\Cref{fig:well-separated} presents a case where $P_0$ and $P_1$ are well-separated, i.e., their supports are separated by two orthogonal hyperplanes. Under this condition, \Cref{Lemma:Only-ne} shows that no MFVI optimizer can simultaneously capture both components.

\begin{lemma}\label{Lemma:Only-ne}
   Assume that the probability measures $P_0$ and $P_1$ are {\it well-separated}. Then for $\mu=\bigotimes_{i=1}^d \mu_i\in \mathcal{P}^{\otimes d}(\R)$ with $\mu\ll\pi$, either $\mu\ll P_0$ or $\mu\ll P_1$ holds. As a consequence, if $\mu^*$ is an MFVI optimizer with $\kl{\mu^*}{\pi}<\infty$, then it holds that
   \[
   \begin{aligned}
   \mu^*&\in \argmin_{\mu \in \mathcal{P}^{\otimes d}(\R)} \kl{\mu}{P_0}\quad\text{or}\quad \mu^*\in \argmin_{\mu \in \mathcal{P}^{\otimes d}(\R)} \kl{\mu}{P_1}.
   \end{aligned}
   \]
\end{lemma}
\begin{proof}
Since $P_0$ and $P_1$ are well-separated, there are distinct coordinates $j,k$ and separating half-spaces as in \cref{defn:epsilon-separated} such that $P_0(H_j^-\cap H_k^-)=P_1(H_j^+\cap H_k^+)=1$.
Since $\mu\ll\pi$, it follows that
\[
\mu(H_j^-\cap H_k^-)+\mu(H_j^+\cap H_k^+)=1.
\]
Suppose both terms are positive. Then $\mu(H_j^+)>0$ and
$\mu(H_k^-)>0$, so that
\begin{equation*}
\begin{aligned}
\mu(H_j^+\cap H_k^-)
&=\mu_j(\{r\in\R:s_jr>b_j\})\mu_k(\{r\in\R:s_kr<b_k\})>0.
\end{aligned}
\end{equation*}
However, $\pi(H_j^+\cap H_k^-)=0$, contradicting $\mu\ll\pi$. Thus $\mu$ gives full mass to exactly one of the two component regions. For the second claim, $\kl{\mu^*}{\pi}<\infty$ implies $\mu^*\ll\pi$. If $\mu^*\ll P_0$, then $\kl{\mu^*}{\pi}=\kl{\mu^*}{P_0}-\log w$; if $\mu^*\ll P_1$, then $\kl{\mu^*}{\pi}=\kl{\mu^*}{P_1}-\log(1-w)$. This concludes the claim.
\end{proof}
 In practice, one often encounters mixture distributions that do not satisfy the exact separation of supports, such as the Gaussian-mixture example in \Cref{fig:mode-collapsing-Gaussian}. In such cases, we use $\vae$-separation to \emph{quantify} the degree of separation between components. 

Furthermore, \Cref{Lemma:Only-ne} generalizes to the setting where both components are $\vae$-separated for some sufficiently small $\vae>0$. To the best of our knowledge, \cref{thm:MFVI-concentration} (see \cref{proof-thm-MFVI-concentration} for its proof) provides the first theoretical explanation of mode collapse: the mass that MFVI assigns to one component can be made arbitrarily small when the two components are sufficiently far apart. Importantly, this result is \emph{algorithm-free}: any method that accurately computes an MFVI optimizer will necessarily exhibit this phenomenon.

\begin{theorem}\label{thm:MFVI-concentration}
Let $P_0,P_1$ be $\vae$-separated and
$\pi=wP_0+(1-w)P_1\in\cP_{\rm ac}(\R^d)$, with $w\in(0,1)$.
Suppose $0<\vae\leq e^{-2b}$, where
\[
 b:=\log(2)+\inf_{\mu\in\cP(\R)^{\otimes d}}\kl{\mu}{\pi}.
\]
Then every MFVI optimizer
$\mu^*=\bigotimes_{i=1}^d\mu_i^*$ satisfies
\begin{equation}\label{eq:MFVI-upper-bound}
\begin{aligned}
\min\left\{\mu^*(H_j^{-}\cap H_k^{-}),
              \mu^*(H_j^{+}\cap H_k^{+})\right\}
 \leq \frac{1-b/\log(\vae^{-1})-\sqrt{1-2b/\log(\vae^{-1})}}{2},
\end{aligned}
\end{equation}
where $H_j^\pm,H_k^\pm$ are the $\vae$-separating half-spaces.
\end{theorem}
\begin{remark}
The function $t\mapsto(1-t-\sqrt{1-2t})/2$ is increasing on
$[0,1/2]$ and takes values in $[0,1/4]$. Thus, for fixed $b$, the
bound in \eqref{eq:MFVI-upper-bound} is increasing in $\vae$ on
$(0,e^{-2b}]$ and is at most $1/4$. If $b$ remains bounded as
$\vae\downarrow0$, the right-hand-side of \eqref{eq:MFVI-upper-bound} decreases to 0, recovering the result in \cref{Lemma:Only-ne}.
\end{remark}
As a direct consequence of \cref{thm:MFVI-concentration}, the MFVI optimizer falls into one of the two regimes below.
\begin{corollary}
Assume the setting of \cref{thm:MFVI-concentration}. For every MFVI optimizer $\mu^*$, there exists a convex set $S$ such that
\[
\begin{aligned}
\mu^*(S)\leq \frac{1-b/\log(\vae^{-1})-\sqrt{1-2b/\log(\vae^{-1})}}{2},
\end{aligned}
\]
and one of the following holds:
\begin{enumerate}
 \item $P_0(S)\geq1-\vae$ and $P_1(S)\leq\vae$;
 \item $P_1(S)\geq1-\vae$ and $P_0(S)\leq\vae$.
\end{enumerate}
\end{corollary}

Returning to the setting of \cref{fig:mode-collapsing-Gaussian}, \cref{coro:Gaussian-m} (see its proof in \cref{proof-coro-Gaussian-m}) implies that for sufficiently large $m$, any MFVI optimizer collapses to one of the two components. This is unsurprising as the two Gaussian components are $\vae_m$-separated with $\vae_m \to 0$ when $m \to \infty$. The same conclusion holds in $\R^d$ provided there exist two orthogonal hyperplanes approximately separating the components. In the sequel, we use $\Phi$ to denote the CDF of the standard normal distribution on $\R$.
\begin{corollary}\label{coro:Gaussian-m}
Let $\pi=w\cN(-m,I_d)+(1-w)\cN(m,I_d)$, with
$m=(m_1,\ldots,m_d)\in\R^d$ and $w\in(0,1)$, and let $b$ be the constant from
\cref{thm:MFVI-concentration}. Suppose $\delta\in(0,1/4]$ and there
are distinct $j,k\in[d]$ such that
\begin{equation}\label{eq:Gaussian-cdf}
 \Phi(|m_j|)\Phi(|m_k|)
 \geq1-\exp\left\{-\frac{b}{2(\sqrt{\delta}-\delta)}\right\}.
\end{equation}
Then every MFVI optimizer satisfies
\begin{equation}\label{eq:Gaussian-result}
\min\left\{\mu^*(H_j^-\cap H_k^-),
             \mu^*(H_j^+\cap H_k^+)\right\} \leq\delta,
\end{equation}
where, for $\ell\in\{j,k\}$, the half-spaces are
$H_\ell^-=\{x:\operatorname{sgn}(m_\ell)x_\ell<0\}$ and
$H_\ell^+=\{x:\operatorname{sgn}(m_\ell)x_\ell>0\}$.
In particular, \eqref{eq:Gaussian-result} holds if
$\min\{|m_j|,|m_k|\}\geq\Phi^{-1}(\beta^{1/2})$, where
\[
 \beta=1-\exp\left\{
 -\frac{\log(2)-\log\max\{w,1-w\}}
        {2(\sqrt{\delta}-\delta)}\right\}.
\]
\end{corollary}
\begin{remark}
Suppose $m=(\|m\|_2/\sqrt{d})(1,\ldots,1)$. For fixed $w\in(0,1)$ and $\delta\in(0,1/4]$, \eqref{eq:Gaussian-result} holds whenever
\[
\|m\|_2 \geq \sqrt{d}\,\Phi^{-1}(\beta^{1/2}).
\]
Thus, in this setting, mode collapse occurs when the mean vectors are separated on the scale of $O(\sqrt{d})$.
\end{remark}

\subsection{Instability of the MFVI Optimizer}\label{sec:instability-mfvi}
In this subsection, we show that MFVI is unstable when $P_0,P_1$ are well-separated.  Define $$ \mathcal{A}_i=\{ \mu\in \mathcal{P}^{\otimes d}(\R): \ \mu\ll P_i\}, \quad i\in \{0,1\}.$$ By \cref{Lemma:Only-ne}, it must follow that 
\begin{align*}
\kl{\mu^*}{\pi}
&= \indic{\mu^*\in\mathcal{A}_0}\left(\int \log\!\left(\frac{\dd\mu^*}{\dd P_0}\right)\dd\mu^*-\log w\right) \\
& + \indic{\mu^*\in\mathcal{A}_1}\left(\int \log\!\left(\frac{\dd\mu^*}{\dd P_1}\right)\dd\mu^*-\log(1-w)\right).
\end{align*}
Consequently, as the mixing weight $w$ varies, the MFVI optimizer exhibits phase transition.
\begin{corollary}\label{coro:Sufficient-condition}
      Assume that the probability measures $P_0$ and $P_1$ are {\it well-separated}. 
      For $i\in\{0,1\}$, define $a_i
:=
\inf_{\mu\in\mathcal P(\R)^{\otimes d}}
H(\mu\mid P_i)
\in[0,\infty]$,
and assume that $\min\{a_0,a_1\}<\infty$. Set $         w^\star := 1/\left(1+\exp(a_1-a_0) \right)$. 
      Then the following hold:
      \begin{enumerate}
          \item If $w > w^*$,  $\mu^* \in \argmin_{\mu \in  \mathcal{P}^{\otimes d}(\R)}  H(\mu \mid P_0).$
 \item If $w < w^*$,  $\mu^* \in \argmin_{\mu \in  \mathcal{P}^{\otimes d}(\R)}  H(\mu \mid P_1). $
      \end{enumerate}
\end{corollary}
\begin{remark}
\begin{enumerate}[label = (\roman*)]
\item When $w= w^*$, it is easy to see that MFVI has at least two solutions --- one in each of the supports.  Moreover, if $P_0,P_1$ are product measures, then  $w^* = \frac{1}{2}$.
\item \Cref{coro:Sufficient-condition} suggests that MFVI experiences a \emph{phase transition} when approximating a bimodal mixture.  This seems to be related to the statistical physics literature on phase transition and the ``spontaneous symmetry breaking" phenomenon. For example, MFVI for the Curie-Weiss model yields two distinct optimizers in the low temperature regime but a unique mean-field MFVI optimizer after the temperature passes a critical threshold, as explained in \cite{Eldan2020}. 
\item The observed phenomenon relates to the bias amplification in machine learning~\cite{bachoc2025majority}. In particular, when the mixing weight is close to 1, \cref{coro:Sufficient-condition} implies that MFVI ignores the minority component $P_1$ entirely.
\end{enumerate}
\end{remark}

\section{Extension to Multi-modal Mixtures}\label{sec:multimodal}
Let $K\geq 2$ and let $P_0,\ldots,P_{K-1}$ be a family of pairwise $\vae$-separated probability measures. Fix the
separating half-spaces $H_{j(r,s)}^\pm,H_{k(r,s)}^\pm$ for each pair
$0\le r<s\le K-1$.
For the reversed order $(s,r)$, use the same
half-spaces with the signs exchanged. More explicitly, $H_{j(s,r)}^{\pm} = H_{j(r,s)}^{\mp}$ and likewise for $k(r,s)$. Define
\[
    R_r=\bigcap_{s\ne r} (H_{j(r,s)}^{-} \cap H_{k(r,s)}^{-}),
    \qquad S=\bigcup_{r=0}^{K-1}R_r.
\]
The region $R_r$ consists of points lying on the side of the component $P_r$ in every pairwise comparison. The following theorem extends \cref{thm:MFVI-concentration} to the multi-modal case.

\begin{theorem}\label{thm:multimodal}
Let $K\ge2$ and let $P_0,\ldots,P_{K-1}$ be pairwise
$\vae$-separated probability measures on $\R^d$. Let $\pi=\sum_{r=0}^{K-1} w_rP_r$ be absolutely continuous, where $w_r>0$
and $\sum_{r=0}^{K-1} w_r=1$. Suppose further that
\begin{equation}\label{eq:strong-separation}
 0<(K-1)\vae\le\exp\left\{-\frac{K+2}{2}b\right\},
\end{equation}
where we recall that $b:=\log(2)+\inf_{\mu\in\cP(\R)^{\otimes d}}\kl{\mu}{\pi}$. Then every MFVI optimizer $\mu^*=\bigotimes_{\ell=1}^d\mu_\ell^*$
satisfies
\begin{equation}\label{eq:square-root}
     \min_{0\le r\le K-1}\sum_{s\ne r}\mu^*(R_s)
\le\frac{1-\tau_\vae-
 \sqrt{(1-\tau_\vae)^2-(K-1)\tau_\vae^2}}2,
\end{equation}
where $ \tau_\vae:=\frac{b}{\log(((K-1)\vae)^{-1})}\le\frac{2}{K+2}$.
\end{theorem}
For each MFVI optimizer, the bound controls the total mass in all component regions except one with maximal mass. When $K=2$, the two regions are $R_0=H_j^{-}\cap H_k^{-}$ and $R_1=H_j^{+}\cap H_k^{+}$, and the right-hand side of~\eqref{eq:square-root} reduces to $(1-\tau_\vae-\sqrt{1-2\tau_\vae})/2$, with $\tau_\vae=b/\log(\vae^{-1})$, recovering the result in~\cref{thm:MFVI-concentration}.

By~\cref{thm:multimodal}, mode-collapse occurs whenever the ratio
$\tau_\vae =b/\log(((K-1)\vae)^{-1})\to0$, since the upper bound
in~\eqref{eq:square-root} then tends to zero. This ratio compares the optimal mean-field
approximation error with the
penalty for placing mass outside $S=\bigcup_{r=0}^{K-1}R_r$.
A small ratio ensures that every MFVI optimizer assigns little mass
to $S^c$. Since spreading mass across distinct components would necessarily create mass outside $S$, the product structure of the optimizers forces concentration in a single
region $R_r$, leading to mode-collapse.

\section{Rotational Variational Inference (RoVI)}\label{sec:rotation-vi}
In \Cref{sec:mode-collapse}, we show that when the two mixture components of $\pi$ are $\vae$-separated for some sufficiently small $\vae>0$, MFVI collapses to one of the components. To tackle this issue, we introduce an ``optimal" change of coordinate to the MFVI problem. Developing a general remedy for mode collapse in mixtures with more than two components is left for future work.

As a motivating example, consider $\pi = \frac{1}{2}\,\cN\left([-m,-m]^\top, I_2
\right) + \frac{1}{2}\,\cN\left([m,m]^\top, I_2
\right)$. Let $O =\tfrac{1}{\sqrt{2}}
\begin{bmatrix}
1 & -1 \\[3pt]
1 &  1
\end{bmatrix}$. Then $\pi =  O_\# \pi_{\text{MF}}$ where $\pi_{\text{MF}}$ is the product measure given by
\[
\pi_{\rm MF}
=\left(\frac{1}{2}\,\cN\left(-\sqrt{2} m, 1\right)
 + \frac{1}{2}\,\cN\left(\sqrt{2} m, 1\right) \right)
 \otimes \cN(0,1) \in \cP(\R)^{\otimes 2}.
\]
As a result, $\pi$ can be recovered by jointly optimizing over the rotation matrix $O$ and product measure $\pi_{\rm MF}$ (see \Cref{fig:GM-comparison}(c) and \Cref{fig:GM-comparison}(d) for illustrations). 
 Motivated by the above, we consider \emph{rotational variational inference (RoVI)}, which approximates $\pi$ within the set of all rotated product measures $\{O_\sharp \mu \mid O\in\cO(d),\, \mu\in\cP_{2,{\rm ac}}(\R)^{\otimes d}\}$ where $\cO(d)$ is the set of orthogonal matrices in $\R^d$, i.e., any $O \in \cO(d)$ is a $\R^{d \times d}$ matrix satisfying $O^\top = O^{-1}, O^{\top} O = I_d$. Specifically, we solve the following minimization:
\begin{equation} \label{RoVI-minimizers}\tag{RoVI}
\left(\cO^\star, \mu^\star \right) \in \argmin_{O\in \cO(d), \mu \in \cP_{2, {\rm ac}}(\R)^{\otimes d }} \kl{O_\sharp \mu}{\pi}.
\end{equation}
\eqref{RoVI-minimizers} generalizes MFVI:
its optimal KL value is no larger than that of MFVI on the space $\cP_{2,{\rm ac}}(\R)^{\otimes d}$, since rotations include the identity.
Since $\kl{O_\sharp \mu}{\pi}  = \kl{\mu}{ O^\top_\sharp \pi}$, \eqref{RoVI-minimizers} is equivalent to
\begin{equation*}
    \argmin_{O\in \cO(d), \mu \in \cP_{2, {\rm ac}}(\R)^{\otimes d }} \kl{\mu}{O_\sharp^\top \pi}.
\end{equation*}
In a related work, \citet{chen2025rotated} select rotations by maximizing a projected Fisher information and solve RoVI subproblems within an iterative Gaussianization procedure. Here, we jointly optimize a single rotation and a product measure to address mode collapse.

Following the lifting approach introduced in \cite{jiang2025algorithms,Sheng2025Stability}, we construct a polyhedral family of optimal transport maps and propose a gradient-based method to approximate the MFVI optimizer. 

Before that, we recall the formulation of optimal transport and some basic results.
Let $\mu, \nu \in \cP(\R^d)$. 
A measurable map $T: \R^d \to \R^d$ is said to \emph{push forward} $\mu$ to $\nu$, denoted $T_{\#}\mu = \nu$, 
if $\nu(A) = \mu(T^{-1}(A))$ for all Borel sets $A \subseteq \R^d$. 
We call $T$ an \emph{optimal transport (OT) map} from $\mu$ to $\nu$ if it minimizes the quadratic transportation cost
\begin{equation*}
    T \in \argmin_{S_\# \mu = \nu} \int \|S(x) - x\|_2^2 \, \mu(\rd x).
\end{equation*}
The celebrated theorem  states that, when $\mu$ is absolutely continuous, 
the optimal transport map $T$ is unique and is given by the gradient of a convex potential~\citep[see, e.g.,][]{Brenier,Cuesta-Matran,McCann.Robert.95.Duke,Rachev-Ruschendorf-article}. 
We say that an OT map $T: \R^d \to \R^d$ is \emph{separable} if it decomposes coordinate-wise as 
$T = (T_1, \dots, T_d)$, where each $T_i: \R \to \R$ is a one-dimensional OT map. 

Let $\rho =\cN(0,I_d)$.  Our approach to~\eqref{RoVI-minimizers} involves two steps: 

\textbf{Step 1:}  Approximate the set $\cP_{2,{\rm ac}}(\R)^{\otimes d }$ by a parameterized class of pushforward measures $\left\{T_{\theta_\#} \rho \mid \theta \in \Theta \right\}$, where $T_\theta$ is a separable optimal transport map of the form
\begin{equation*}
T_{\theta}(x)
= v + \sum_{T \in \cM} \lambda_{T} T(x),
\end{equation*}
with $\theta = (\lambda, v) \in \Theta:=  \R_+^N \times \R^d $ and $N=|\cM|$. The set $\cM$ is a dictionary of separable OT maps (see \cite[Section~5.3]{jiang2025algorithms} for an explicit construction of $\cM$), and the set $\{T_\theta \mid \theta \in \Theta \}$ is the polyhedral family formed by maps in $\cM$.

\textbf{Step 2:}  Use projected gradient descent to approximate the following minimization: 
\begin{equation} \label{RoVI-alg-minimizers}
    \left(O_\diamond^\star,  \theta^\star_\diamond \right) \in  \argmin_{O\in \cO(d),\theta \in \Theta} \kl{ \left(O \circ T_\theta\right)_\# \rho}{\pi}.
\end{equation}
 A minimizer of the parameterized problem takes the form $\left(O_\diamond^\star \circ T_{\theta_\diamond^\star} \right)_\# \rho$.  
 
 We solve the optimization~\eqref{RoVI-alg-minimizers} using the coordinate gradient descent algorithm. To that end, we explicitly compute the gradient of
\begin{equation*}
     \cO(d) \times \Theta \ni (O,\theta) = (O, \lambda, v) \mapsto \kl{ \left(O \circ T_{\lambda, v}\right)_\# \rho}{\pi}.
\end{equation*}
For any $(O, \lambda, v) \in \cO(d) \times \Theta$, denote by $\mu_{\lambda, v} := (T_{\lambda, v})_\sharp \rho$ and $\pi_{O} := (O^\top)_\sharp \pi$, we know that
\begin{equation}\label{eq:KL}
       \kl{(O\circ T_{\lambda, v})_\sharp \rho}{\pi}
   = \kl{\mu_{\lambda, v}}{\pi_O}
   = \int V\circ O \dd \mu_{\lambda, v}+ H\bigl(\mu_{\lambda,v}\bigr) + \log Z,
\end{equation}
where we recall that $Z$ is the normalizing constant of $\pi$. The gradient updates are available when the objective is finite, $DT_{\lambda,v}\succ0$ $\rho$-almost everywhere, and $V$ is sufficiently integrable to allow differentiation.
\\
\textbf{MFVI step: update $(\lambda, v)$.} 
Let $Q\in\R^{N\times N}$ be a positive semi-definite Gram matrix with entries $Q_{T, \tilde{T}}  :=  \langle T, \tilde{T} \rangle_{L_2(\rho)}$ for $T, \tilde T \in \cM$. After setting a large smoothness constant $L$, we project and discretize the following gradient flow for the gradient descent updates with target $\pi_{O_t} \propto \exp \left( - V \circ O_t \right)$:\footnote{See \Cref{sec:partial-gradient} for the explicit expressions of the partial gradients of $\kl{\mu_{\lambda_t, v_t}}{\pi_{O_t}}$ w.r.t.\ $(\lambda, v)$. In addition, we assume that $Q$ is positive definite.} 
\begin{equation}\label{GF-MFVI}
\dot \lambda_t = - \tfrac{1}{L} Q^{-1} \nabla_\lambda \kl{\mu_{\lambda_t, v_t}}{\pi_{O_t}},
\qquad
\dot v_t = - \tfrac{1}{L} \nabla_v \kl{\mu_{\lambda_t, v_t}}{\pi_{O_t}}.
\end{equation}
The gradient flow enjoys linear convergence when $\pi$ is $L$-log smooth and $\alpha$-log-concave~\cite{lacker2023independent, jiang2025algorithms}. 
\\
\textbf{Rotation step: update $O$.} 
For fixed $\mu = T_\# \rho$.
By \eqref{eq:KL}, the unconstrained gradient of $\kl{\mu}{\pi_{O}}$ w.r.t.\ $O \in \cO(d)$ is
\begin{equation*}
G(O) = \int \nabla V \left( O T(x) \right) \left(T(x) \right)^\top  \rho(\rd x). 
\end{equation*}
The tangent space at $O$ is $T_O\cO(d)= \{O\Omega \in \R^{d\times d}:\ \Omega^\top=-\Omega \}$.
Projecting the unconstrained gradient onto $T_O\cO(d)$ \citep[Eq. 2.4]{Edelman1999} yields
\begin{equation*}
\tilde G(O) = G(O) - O \mathrm{sym} \big(O^\top G(O)\big),
\qquad \mathrm{sym}(A) = \tfrac12(A+A^\top).
\end{equation*}
Thus, the gradient flow for updating $O$ at time $t$ is given by
 \begin{equation} \label{GF-O}
    \dot O_t  = - \tilde G(O_t). 
 \end{equation}
The rotational VI algorithm approximately minimizes \eqref{RoVI-alg-minimizers} by discretizing the gradient flows specified by \eqref{GF-MFVI} and \eqref{GF-O} with some chosen stepsizes; see \Cref{algo:RoVI} for details.
\begin{algorithm}[h]
\caption{Rotational Variational Inference}
\label{algo:RoVI}
\begin{algorithmic}
\State \textbf{Input:} Dictionary $\cM$, Gram matrix $Q$, a smoothness constant $L > 0$, step sizes $\eta_{\text{MF}}, \eta_O > 0$.
\State \textbf{Initialize:} $\theta_0=(\lambda_0,v_0)\in\Theta$ and $O_0\in\cO(d)$ with positive transport Jacobian and finite objective. 
\For{$t=0,1,2,\dots$}
  \State \textbf{MFVI step:} Update $\theta = (\lambda,v)$ by projected gradients with $\pi_{O_t} \propto \exp \left( - V \circ O_t \right)$:
\begin{equation*}
\begin{aligned}
  \lambda_{t+1}
  &= \proj_{(\R_+^N,\|\cdot\|_Q)}\Big(
     \lambda_t - \tfrac{\eta_{\text{MF}}}{L}\, Q^{-1} \nabla_\lambda \kl{\mu_{\lambda_t, v_t}}{\pi_{O_t}}\Big),\\
  v_{t+1} &= v_t - \tfrac{\eta_{\text{MF}}}{L}\, \nabla_v \kl{\mu_{\lambda_t, v_t}}{\pi_{O_t}}.
\end{aligned}    
\end{equation*}
  \State If the candidate has a singular transport Jacobian or infinite objective, set $\theta_{t+1}=\theta_t$.
  \State \textbf{Rotation step:} Compute gradient:
  \[
  G(O_t) = \int \nabla V\!\big(O_tT_{\theta_{t+1}}(x)\big)\,T_{\theta_{t+1}}(x)^\top\,\rho(dx).
  \]
 \State Project to $T_{O_t}\cO(d)$:
  \[
  \tilde G(O_t) = \tfrac{1}{2}G(O_t) - \tfrac{1}{2}O_t \, G(O_t)^\top O_t .
  \]
  \State Perform the gradient update and retract by a QR factorization with positive diagonal in $R$:
  \[
  \widetilde O = O_t - \eta_O \tilde G(O_t), 
  \qquad \widetilde O = \tilde Q R, \quad O_{t+1} = \tilde Q.
  \]
  \State If the rotation update gives an infinite objective, set $O_{t+1}=O_t$.
\EndFor
\end{algorithmic}
\end{algorithm}

To implement the algorithm, we tune the step sizes $\eta_{\text{MF}}$ and $\eta_O$. We find that choosing $\eta_{\text{MF}} = 0.001$ and $\eta_O = 0.01$ provides stable numerical fits. Since the RoVI problem is non-convex in the Euclidean and Wasserstein sense, convergence to a global minimizer is not guaranteed.
We recommend using multiple random initializations and choosing the fitted approximation with the highest estimated ELBO.

The following proposition characterizes the approximation quality of RoVI for a two-component Gaussian mixture with a common covariance (see \cref{proof:prop:rovi-kl-gaussian-mixture} for the proof). 

\begin{proposition}\label{prop:rovi-kl-gaussian-mixture}
Let $P_0=\cN(m_0,\Sigma)$ and $P_1=\cN(m_1,\Sigma)$ with mean vectors $m_0,m_1\in\R^d$ and covariance matrix $\Sigma\in\R^{d\times d}$. Let $\pi:=wP_0+(1-w)P_1$ for $w \in (0,1)$. Then the following hold:
\begin{enumerate}
    \item If $m_0\neq m_1$, set $v_1=\frac{m_1-m_0}{\|m_1-m_0\|_2}$ and choose $v_2,\ldots,v_d$ to be $(d-1)$ orthonormal eigenvectors of $(I_d-v_1v_1^\top)\,\Sigma^{-1}\,(I_d-v_1v_1^\top)$ spanning $v_1^\perp$. Then  \begin{equation}\label{eq:rovi-kl-gm}
  \inf_{\substack{O\in\cO(d),\\ \mu\in\cP(\R)^{\otimes d}}}
  \kl{O_\sharp \mu}{\pi} \leq \frac{1}{2}\left(
    \log(\det(\Sigma))
    + \sum_{i=1}^d \log\left(v_i^\top \Sigma^{-1} v_i\right)
  \right).
\end{equation}
\item Otherwise, $ \inf_{O\in\cO(d),\mu\in\cP(\R)^{\otimes d}}
  \kl{O_\sharp \mu}{\pi} = 0$.
\end{enumerate}
\end{proposition}
When $v_1$ is an eigenvector of $\Sigma$, the bound~\eqref{eq:rovi-kl-gm} reduces to $0$, indicating that the RoVI optimizer exactly recovers the target mixture $\pi$. The upper bound~\eqref{eq:rovi-kl-gm} is explicit and can be evaluated without solving for the RoVI optimizer. 
To compute the bound, it suffices to evaluate the determinant $\det(\Sigma)$, the direction vector $v_1 = (m_1 - m_0)/\|m_1 - m_0\|_2$, and the eigendecomposition of the projected precision matrix  $(I_d - v_1v_1^\top)\,\Sigma^{-1}\,(I_d - v_1v_1^\top)$.  The total computational cost is $O(d^3)$ as determined by the eigendecomposition.
\subsection{Numerical Experiments}
In this numerical study, we provide empirical results to support our theoretical findings. 
Consider a two-component Gaussian mixture with potential
\begin{align*}
V(x)  & = - \log \Bigg(
 w |\Sigma_1|^{-1/2} \exp\Big( -\tfrac{1}{2}
\left\|\Sigma_1^{-1/2}(x-m_1)\right\|_2^2 \Big)\\
&+ (1-w)|\Sigma_2|^{-1/2} \exp\Big( -\tfrac{1}{2}
\left\|\Sigma_2^{-1/2}(x-m_2)\right\|_2^2 \Big)
\Bigg),
\end{align*}
where $x \in \R^2$, $m_1, m_2 \in \R^2$ are the mean vectors, and $\Sigma_1, \Sigma_2 \in \R^{2 \times 2}$ are the covariance matrices. The parameter $w \in (0,1)$ specifies the mixing weight of the first Gaussian component. The algorithms to compare are Langevin Monte Carlo (LMC), MFVI, and RoVI.  
We compare LMC, MFVI, and RoVI across six Gaussian targets with varying mixing weights, mean locations, and covariance matrices (\Cref{fig:GM-comparison}). The target distributions are given in Appendix~\ref{sec:gaussian-experiment-settings}.

\begin{figure}[p]
  \centering
  \begin{subfigure}[t]{0.45\linewidth}
    \includegraphics[width=\linewidth, trim=0 0 0 0, clip]{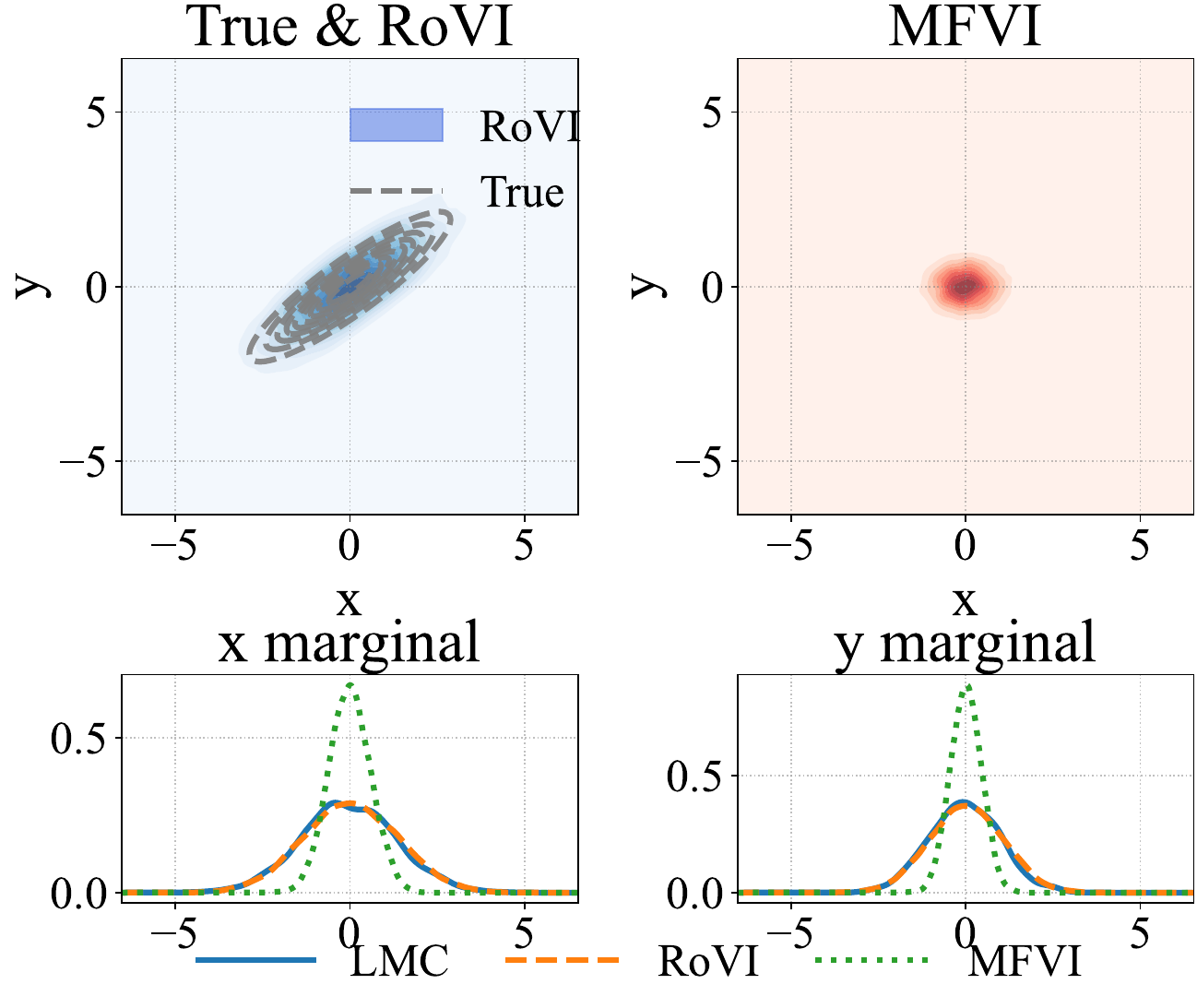}
    \caption{Correlated Gaussian}
    \label{fig:single-Gaussian}
  \end{subfigure}\hspace{0.0em}
  \begin{subfigure}[t]{0.45\linewidth}
    \includegraphics[width=\linewidth, trim=0 0 0 0, clip]{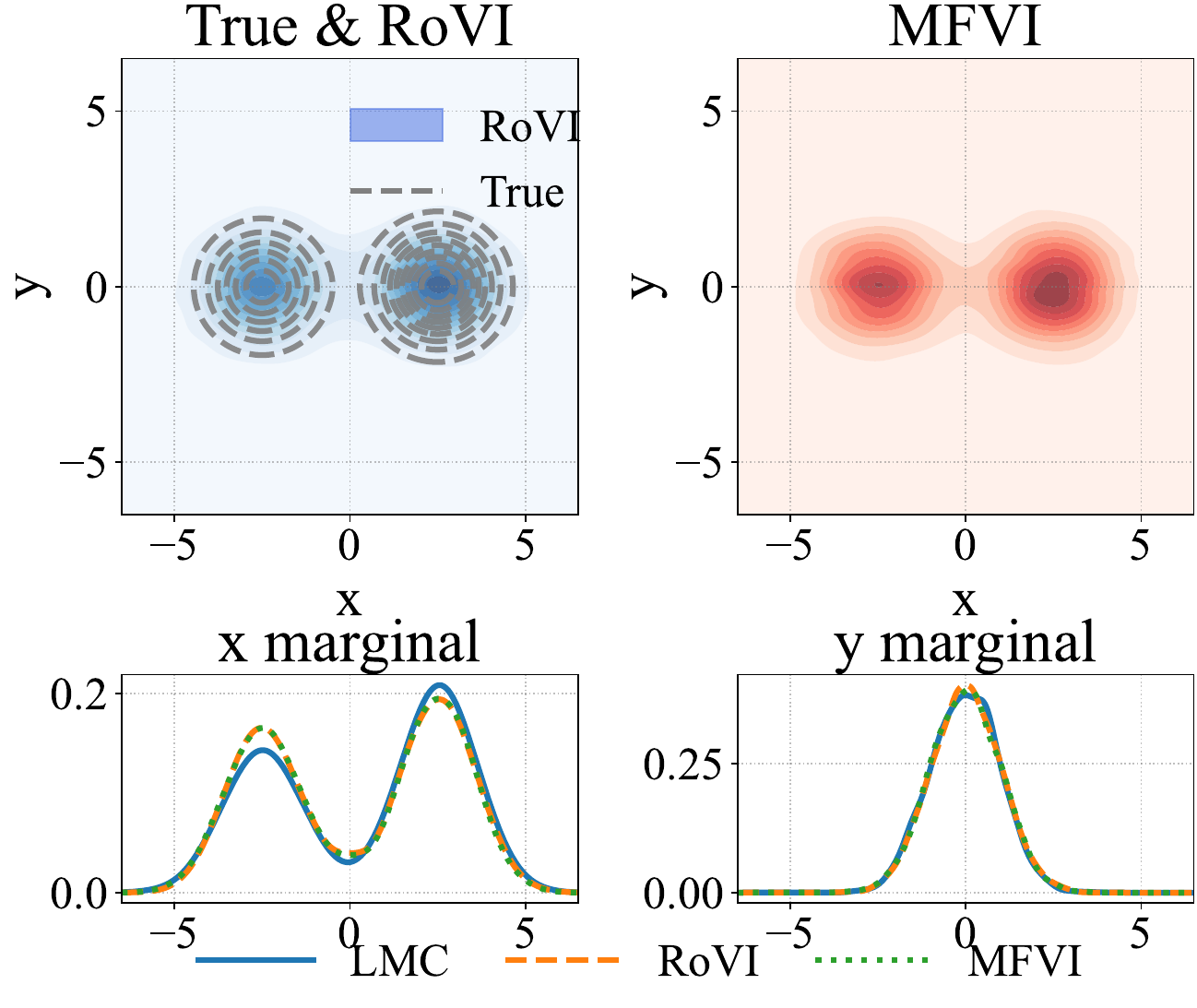}
    \caption{Coordinate-aligned modes}
    \label{fig:aligned-GM}
  \end{subfigure}

  \vspace{0.0em}

  \begin{subfigure}[t]{0.45\linewidth}
    \includegraphics[width=\linewidth, trim=0 0 0 0, clip]{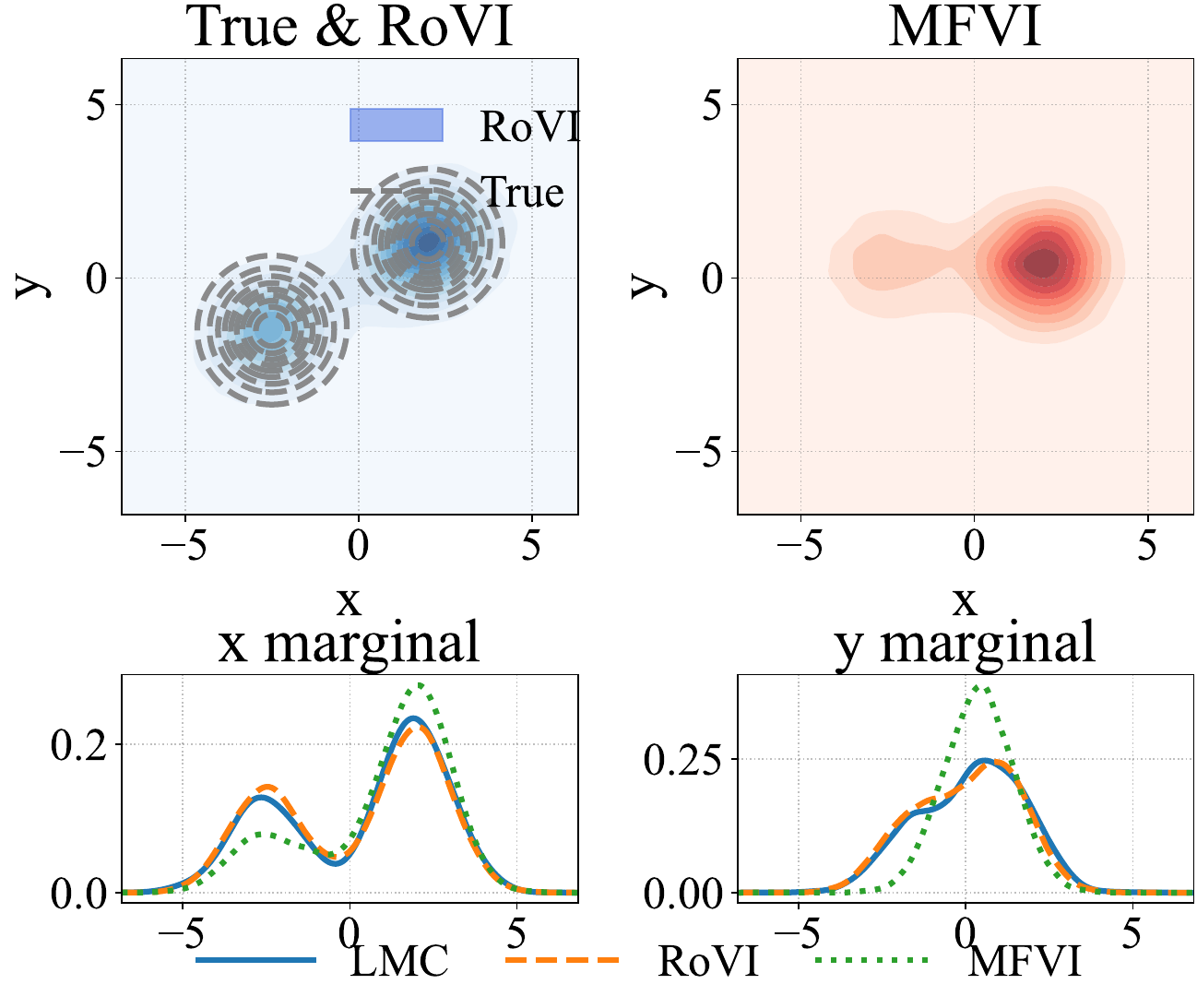}
    \caption{Misaligned modes}
    \label{fig:misaligned-GM}
  \end{subfigure}\hspace{0.0em}
  \begin{subfigure}[t]{0.45\linewidth}
    \includegraphics[width=\linewidth, trim=0 0 0 0, clip]{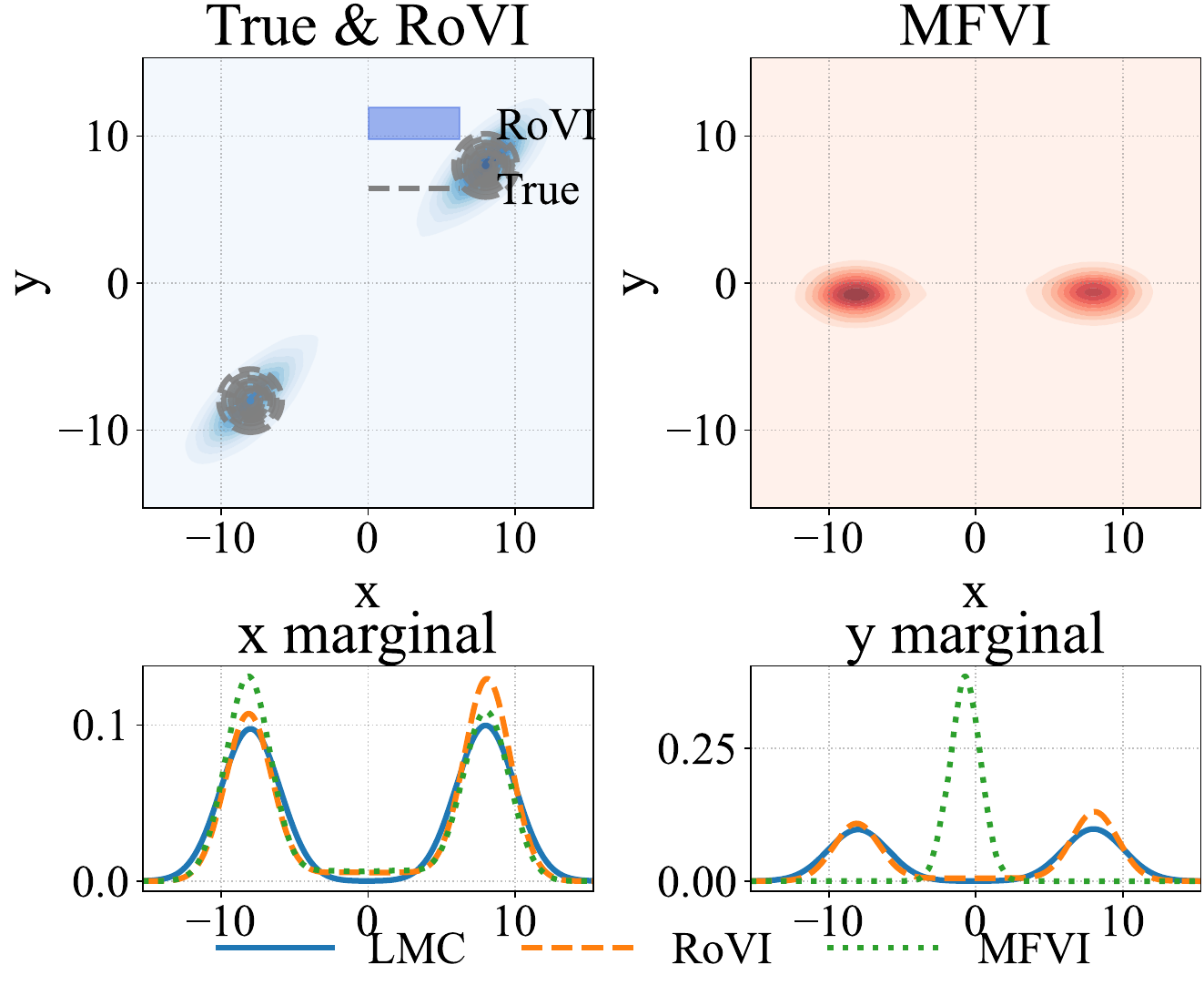}
    \caption{Far-away modes}
    \label{fig:far-GM}
  \end{subfigure}

  \vspace{0.0em}

  \begin{subfigure}[t]{0.45\linewidth}
    \includegraphics[width=\linewidth, trim=0 0 0 0, clip]{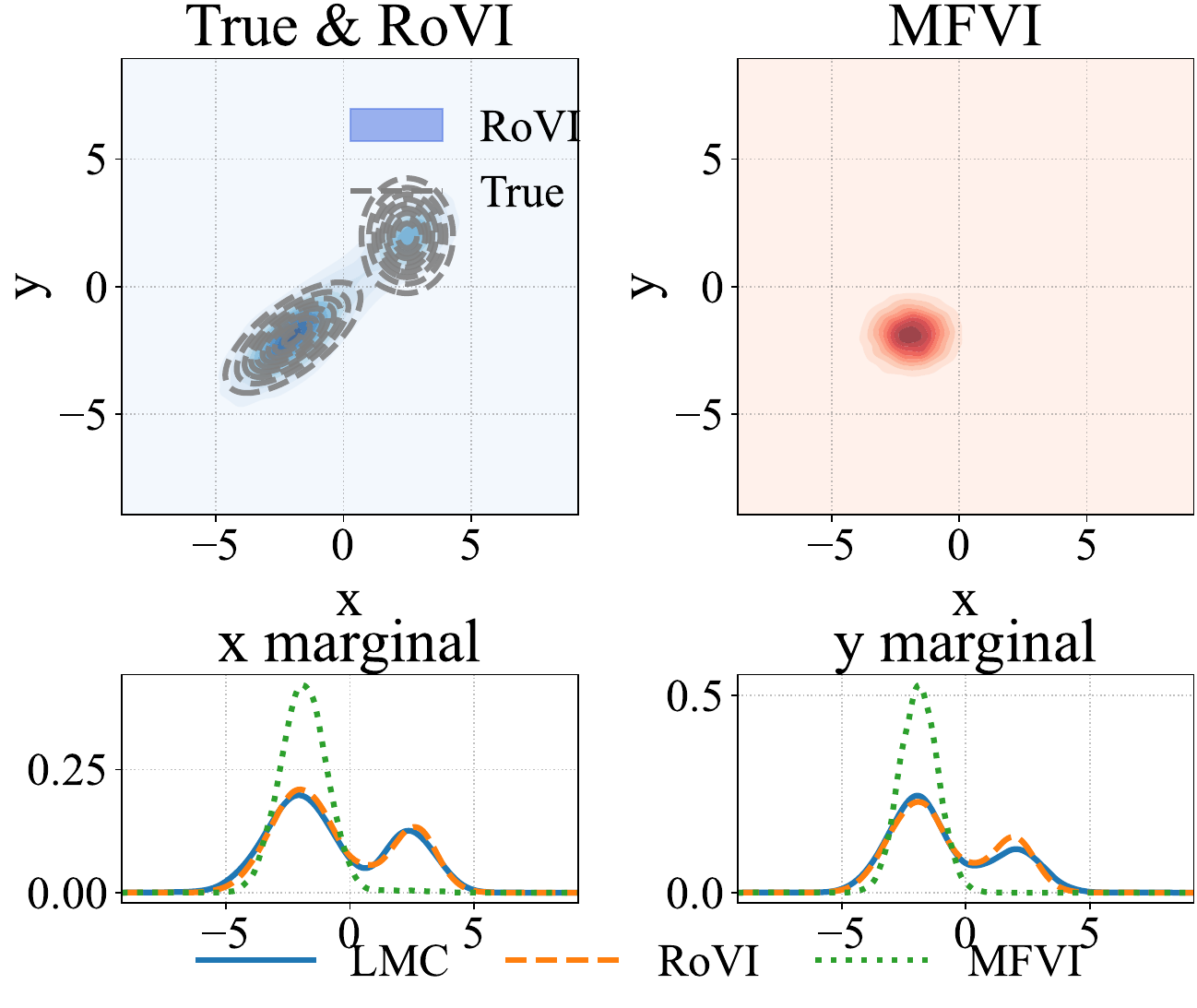}
    \caption{Covariance asymmetry}
    \label{fig:asymmetry-GM}
  \end{subfigure}\hspace{0.0em}
  \begin{subfigure}[t]{0.45\linewidth}
    \includegraphics[width=\linewidth, trim=0 0 0 0, clip]{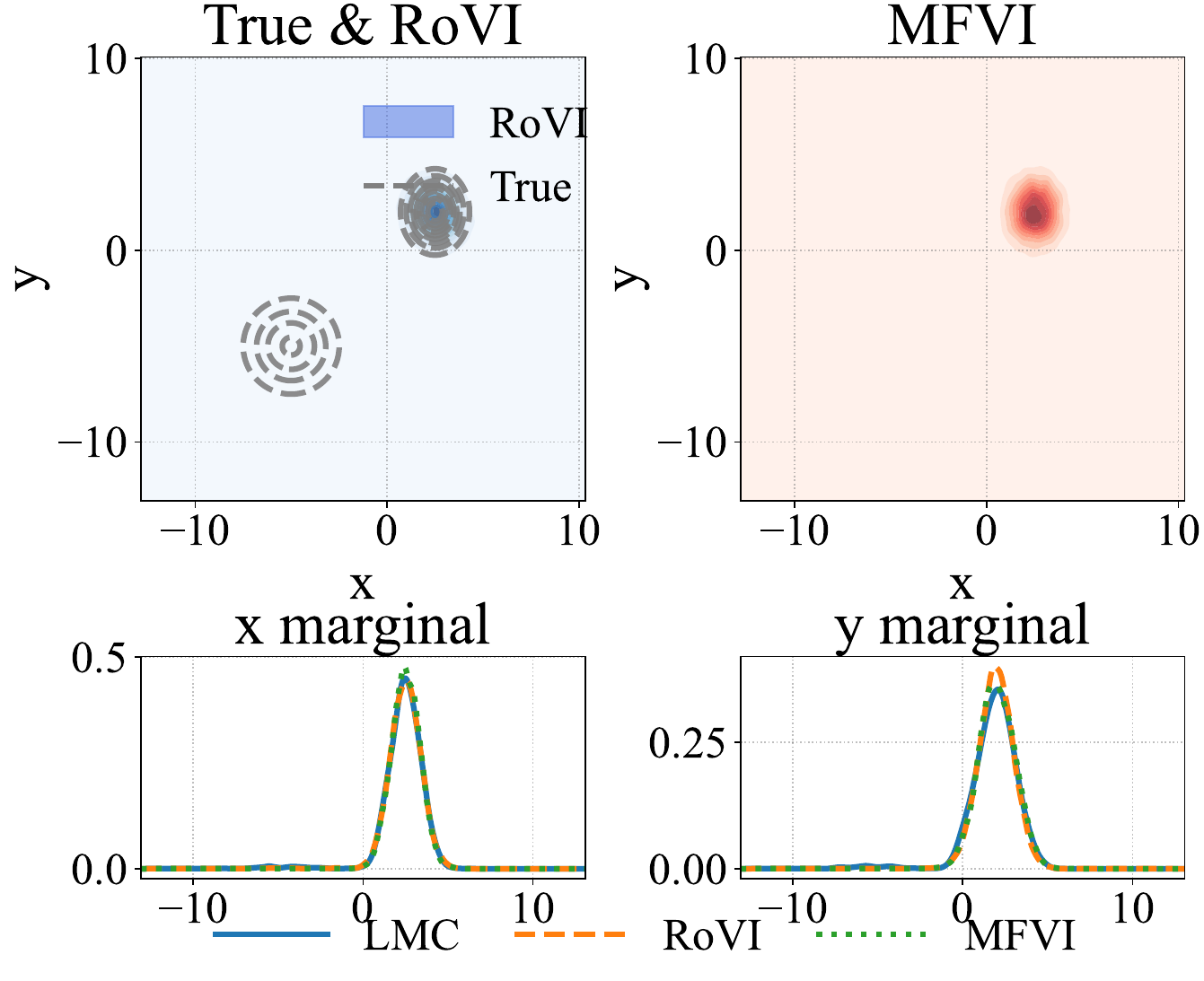}
    \caption{Far-away modes, covariance asymmetry}
    \label{fig:asymmetry-far-GM}
  \end{subfigure}

  \caption{
  Comparison of LMC, MFVI, and RoVI across six Gaussian mixture settings with varying means and covariances. 
  RoVI improves approximation in several settings where MFVI is underdispersed or collapses.}
  \label{fig:GM-comparison}
\end{figure}
See \Cref{fig:MFVI-mode-collapse-RoVI-works} for the additional results on the performance of RoVI on learning a symmetric Gaussian mixture, compared to LMC and MFVI.
\begin{figure}[!htbp]
\centering
\includegraphics[width=\linewidth]{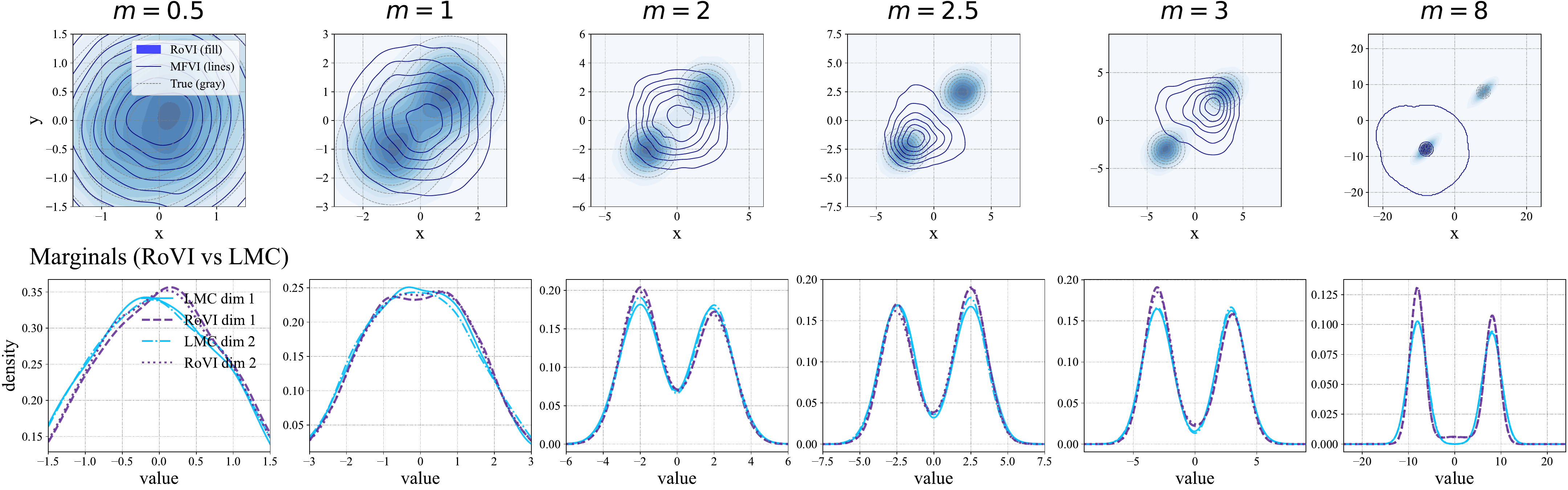}
\caption{
  The target is a symmetric Gaussian mixture:
  $ \frac{1}{2}\,\cN\left([-m,-m]^\top, I_2
\right) + \frac{1}{2}\,\cN\left([m,m]^\top, I_2
\right)$.
  RoVI captures both components and the marginal distributions, while MFVI collapses to a single mode in the sufficiently separated cases.
}
\label{fig:MFVI-mode-collapse-RoVI-works}
\end{figure}

\subsection{Application to Other Posteriors}
\label{sec:posterior-mode-retention}

We evaluate MFVI and RoVI on two Bayesian mixture posteriors. The first is \texttt{normal\_2-normal\_mixture} from \href{https://github.com/stan-dev/posteriordb/tree/91027f9cdd13b7e76283fec3f2898f59d4b83ec4/posterior_database}{PosteriorDB} \citep{magnusson2025posteriordb}. The model consists of 1,000 observations, three parameters, and known component variances. 
The second is a real-data example, where we fit a five-parameter Gaussian mixture to the 272 \href{https://stat.ethz.ch/R-manual/R-devel/library/datasets/html/faithful.html}{Old Faithful eruption durations} \citep{hardle1991smoothing}, standardized by their sample mean and standard deviation. 
We place independent priors on the parameters: $\mu_k \sim \mathcal{N}(0, 2^2)$, $\log\sigma_k \sim \mathcal{N}(0, 1)$, and $w \sim \operatorname{Beta}(5,5)$. 
Due to label switching, both posteriors contain symmetric modes with equal probabilities.

We use coordinatewise monotone rational-quadratic splines and jointly optimize an orthogonal matrix for RoVI. This implementation and \Cref{algo:RoVI} optimize the same objective with different transport-map parameterizations. See Appendix~\ref{sec:posterior-experiment-settings} for the training details.

For the exact posterior, we run 4 NUTS chains restricted to the label region $A = \{\mu_1 < \mu_2\}$, followed by swapping the component labels to represent both symmetric modes. 
We assess mode coverage via the minority-label mass $m(q) = \min\{q(A), q(A^c)\}$ and local mass $\ell(q) = \min\{q(E_A \cap A), q(E_{A^c} \cap A^c)\}$, where each full-dimensional ellipsoid $E_B$ encompasses 95\% of the conditional reference draws in region $B$. 
The exact posterior achieves $m = 0.5$ and $\ell = 0.475$, where the local mass quantifies the concentration around both modes simultaneously.

Across both targets, RoVI recovers both modes across all five runs, whereas MFVI collapses to a single mode (\Cref{tab:posterior-mode-retention,fig:posterior-mode-retention}). 
Under RoVI, the mean minority-label mass is $0.267$ for the simulated target and $0.302$ for Old Faithful. 
For Old Faithful, RoVI achieves a mean local-region mass of $0.285$.  
The separated-mixture density plots appear in \Cref{fig:posterior-separated-mixture} in Appendix~\ref{sec:posterior-experiment-settings}.

\begin{table}[t]
\centering
\small
\setlength{\tabcolsep}{4pt}
\begin{tabular}{llcc}
\hline
Target & Method & $m(q)$ & $\ell(q)$ \\
\hline
3D & MFVI & $0.000 \pm 0.000$ & $0.000 \pm 0.000$ \\
3D & RoVI & $0.267 \pm 0.098$ & $0.252 \pm 0.091$ \\
Old Faithful & MFVI & $0.000 \pm 0.000$ & $0.000 \pm 0.000$ \\
Old Faithful & RoVI & $0.302 \pm 0.137$ & $0.285 \pm 0.131$ \\
\hline
\end{tabular}
\caption{Comparison of mode coverage for the simulated mixture (3D)
and Old Faithful posterior. Reported values are means $\pm$ standard
deviations across five paired runs following validation-based
restart selection.}
\label{tab:posterior-mode-retention}
\end{table}

\begin{figure}[h]
\centering
\includegraphics[width=\linewidth]{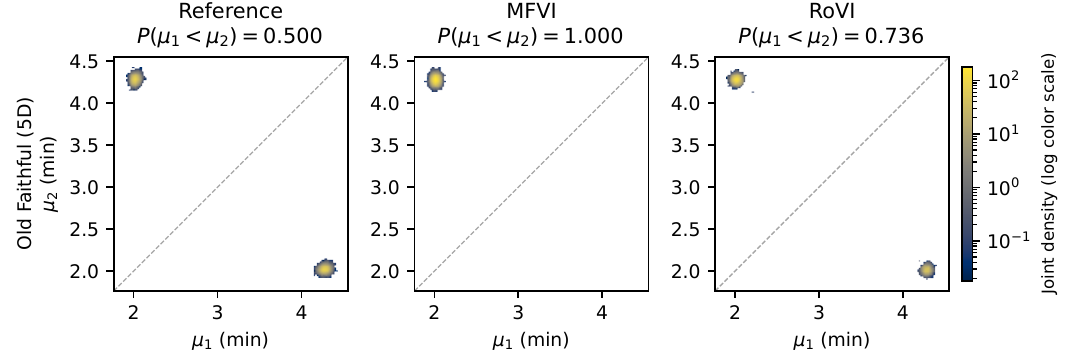}
\caption{Joint densities for the Old Faithful posterior. We also report the probability $P(\mu_1 < \mu_2)$. The columns show the reference and the fitted MFVI and RoVI approximations. Values of $P(\mu_1 < \mu_2)$ closer to $0.5$ indicate more balanced coverage of the two symmetric modes.}
\label{fig:posterior-mode-retention}
\end{figure}

\section{Discussion} \label{sect-discussion}

In this paper, we developed theoretical results to explain the mode-collapsing phenomenon of mean-field variational inference (MFVI), and proposed a solution via \emph{rotational variational inference} (RoVI). Several extensions of this work are of interest to pursue.

\textbf{Theory of RoVI:} A rigorous theoretical analysis of RoVI remains an open challenge. One could aim to study the number of iterations needed for the RoVI algorithm to converge and derive general approximation guarantees under various assumptions on the target. 

\textbf{Effect of $\alpha$-divergence:} A natural direction is to investigate how the choice of $\alpha$-divergence affects mode collapse in MFVI. In particular, is there a choice of $\alpha$ such that the optimizers could preserve multiple components in the target distribution?

\acknowledgments{The authors thank Mathis Deronzier and Lucas De Lara for helpful discussions in formulating the problem, and Yifan Chen for insightful comments. In particular, \cref{Lemma:Only-ne} resulted from discussions with Mathis.}

\section*{AI Use Statement}
The research underlying Sections~2 and~4 (except Section~4.2) was essentially completed by the authors
in October~2025 with almost no AI assistance. Sections~3 and 4.2 were
developed with substantial assistance from AI tools. The authors take full responsibility for the content and conclusions of this paper.

\bibliographystyle{abbrvnat-nodoiurl}
\bibliography{ref}

\clearpage
\appendix

\section{Proof of Results in \modecollapsesectionref}\label{sec:proof-of-mode-collapse}
The following lemma is used to prove~\cref{thm:MFVI-concentration}. 
\begin{lemma}\label{lemma:KL-lower-bound}
Let $\pi\in\cP_{\rm ac}(\R^d)$, let $S\subseteq\R^d$ be measurable
with $0<\pi(S)<1$, and let $\delta\in[0,1]$. Then
\[
 \inf_{\substack{\mu\ll\pi,\ \mu\in\cP(\R^d)\\\mu(S)=\delta}}
 \int_S\log\left(\frac{\dd\mu}{\dd\pi}\right)\,\dd\mu
 =\delta\log\left(\frac{\delta}{\pi(S)}\right),
\]
with $0\log0=0$. For $\delta>0$, the minimum is attained if and only if
$\mu|_S=(\delta/\pi(S))\pi|_S$.
\end{lemma}
\begin{proof}
    If $\delta=0$, the integral is zero. Suppose $\delta>0$. Since $\mu\ll\pi$, we have $\mu\in\cP_{\rm ac}(\R^d)$. We also use the same letter to denote the measure and its density for brevity. Define $\mu_S = \frac{1}{\delta}\mu{\rm 1}_S$, then $\mu_S\in \cP_{\rm ac}(\R^d)$ and
\begin{equation*}
        \int_S \log \left(\frac{\mu}{\pi} \right)\,\dd \mu = \delta \log (\delta) + \delta \int \log \left(\frac{\mu_S}{\pi} \right) \,\dd \mu_S.
\end{equation*}
    By Jensen's inequality, we know that
\begin{equation*}
     \int \log \left(\frac{\mu_S}{\pi}\right) \,\dd \mu_S = - \int \log\left( \frac{\pi}{\mu_S} \right) \,\dd \mu_S \geq -\log \left( \int \frac{\pi}{\mu_S} \,\dd \mu_S \right)  \geq -\log \left(\pi(S) \right).
\end{equation*}
    Combining with the above display shows that the minimum is achieved precisely when $\mu|_S=(\delta/\pi(S))\pi|_S$.
\end{proof}
Now we are ready to prove the main result.
\begin{proof}[Proof of \cref{thm:MFVI-concentration}]\label{proof-thm-MFVI-concentration}
    Without loss of generality, we assume that $P_0,P_1$ are $\vae$-separated by the hyperplanes $H_1^\pm,H_2^\pm$ defined by $s_1,s_2, b_1,b_2\in \R$ with $s_1= s_2 =1$. Since $\kl{\mu^*}{\pi}<\infty$ and $\pi\in\cP_{\rm ac}(\R^d)$, the boundaries of the half-spaces are $\mu^*$-null. Denote by
\[
\mu_i^*((-\infty,b_i]) = a_i^-,\,\, \mu_i^*((b_i,+\infty) )= a_i^+, \quad i\in \{1,2\},
\]
then $a_i^+ + a_i^- = 1$ and
\[
\mu^*(H_1^{+}\cap H_2^{+}) = a_1^+a_2^+\quad{\rm and} \quad \mu^*(H_1^{-}\cap H_2^{-}) = a_1^-a_2^-.
\]
Set $c=\min\{a_1^+a_2^+,a_1^-a_2^-\}$. If $c=0$, the conclusion is immediate. Otherwise, we have
\[
1= a_1^+ + a_1^- \geq c\left(\frac{1}{a_2^+} + \frac{1}{a_2^-} \right)= c\left(\frac{a_2^++a_2^-}{a_2^+ a_2^-} \right).
\]
As $a_2^- = 1- a_2^+$, we get $c\leq a_2^+(1-a_2^+)\leq 1/4$.
By symmetry, it suffices to consider $c=a_1^-a_2^-$. Defining $c_0=\sqrt{c}\in(0,1/2]$, the arithmetic--geometric mean inequality gives
\begin{equation}\label{eq:optimizer-lower-bound}
 a_1^-+a_2^-\geq2\sqrt{a_1^-a_2^-}=2c_0.
\end{equation}
Define $S = \left(H_1^{+}\cap H_2^{+}\right) \cup \left(H_1^{-}\cap H_2^{-}\right)$, $\eta:=\pi(S^c)\leq \vae$ and $\delta:=\mu^*(S^c)$. As $P_0,P_1$ are $\vae$-separated, $\eta\leq \vae$. Furthermore,   \eqref{eq:optimizer-lower-bound} implies (recall that the boundaries of the half-spaces are negligible wrt $ \mu^*\ll \pi$)
\begin{align}
\begin{split}
    \label{eq:bound-delta}
 \delta &= \mu^*\left( \left(H_1^{-}\cup H_2^{-}\right) \cap \left(H_1^{+}\cup H_2^{+}\right)\right)
     = \mu^*\left( \left(H_1^{-}\cap H_2^{+}\right) \cup  \left(H_1^{+}\cap H_2^{-}\right)\right)\\
    &= \mu^* \left(H_1^{-}\cap H_2^{+}\right) +\mu^*  \left(H_1^{+}\cap H_2^{-}\right)
     = a_1^-+a_2^--2c\geq 2c_0(1-c_0).
\end{split}
\end{align}
If $\eta=0$, absolute continuity gives $\delta=0$, and \eqref{eq:bound-delta} forces $c=0$. Thus we may assume $\eta>0$. Then \cref{lemma:KL-lower-bound} implies
\begin{align*}
    \kl{\mu^*}{\pi} &= \int_{S} \log \left(\frac{\dd \mu^*}{\dd \pi} \right)\,\dd \mu^* +  \int_{S^c} \log \left(\frac{\dd \mu^*}{\dd \pi} \right)\,\dd \mu^*\\
    & \geq (1-\delta)\log\left(\frac{1-\delta}{1-\eta} \right) +\delta  \log \left(\frac{\delta}{\eta} \right).
\end{align*}
As the function  $(0,1]\ni x\mapsto a\log \left(\frac{a}{x} \right)$ for $a>0$ is decreasing and $\eta \in (0,\vae]$, it follows that
\[
 (1-\delta)\log\left(\frac{1-\delta}{1-\eta} \right) +\delta  \log \left(\frac{\delta}{\eta} \right) \geq (1-\delta)\log(1-\delta) + \delta \log\left( \frac{\delta}{\vae} \right).
\]
Furthermore, since $  x\log( x ) + (1-x)\log(1-x)\geq  -\log(2) $ holds for all $x\in [0,1],$  we get
\begin{align*}
  (1-\delta)\log(1-\delta) + \delta \log\left( \frac{\delta}{\vae} \right)&\geq-\log(2)+ \delta \log \left(\vae^{-1} \right)\\
  &\geq -\log(2)+ 2c_0(1-c_0)\log \left(\vae^{-1} \right),
\end{align*}
where we used
 $\delta \in [2c_0(1-c_0),1]$ and $\vae\in (0,1)$. Hence, we have shown that
 \begin{equation}\label{proof-modes-collapes-KL}
      \kl{\mu^*}{\pi}\geq- \log(2) +2c_0(1-c_0)\log \left(\vae^{-1} \right) .
 \end{equation}
Recall that $b:=\log(2)+\inf_{\mu\in\cP(\R)^{\otimes d}}\kl{\mu}{\pi}$ and $t=b/\log(\vae^{-1})\leq1/2$. Now we argue by contradiction. Assume that
\[
c>\frac{1-t-\sqrt{1-2t}}{2}
 =\left(\frac{1-\sqrt{1-2t}}{2}\right)^2.
\]
Then $\frac12\geq c_0>\frac{1-\sqrt{1-2t}}2$, so that $2c_0(1-c_0)>t$. Hence \eqref{proof-modes-collapes-KL} yields
\[
\kl{\mu^*}{\pi}\geq-\log(2)+2c_0(1-c_0)\log(\vae^{-1})
 >b-\log(2)=\inf_{\mu\in\cP(\R)^{\otimes d}}\kl{\mu}{\pi},
\]
which contradicts the fact that $\mu^*$ is optimal.
\end{proof}
\cref{lemma:kl-mixture-bound} upper bounds the KL divergence between any component of the mixture and the mixture distribution and is used to establish \cref{coro:Gaussian-m}.
\begin{lemma}\label{lemma:kl-mixture-bound}
Let $\mu,\nu\in\cP_{\rm ac}(\R^d)$, $w\in(0,1)$, and
$\pi=w\mu+(1-w)\nu$. Then
\[
 0\leq\kl{\mu}{\pi}\leq-\log(w),\qquad
 0\leq\kl{\nu}{\pi}\leq-\log(1-w).
\]
\end{lemma}
\begin{proof}
Clearly, $\mu\ll \pi$ and both lower bounds follow from nonnegativity of relative entropy. As $\log(1+x)\geq 0$ for all $x\geq 0$, we have
    \begin{align*}
        \kl{\mu}{\pi} =& \int \log \left(\frac{\mu}{w\mu+(1-w)\nu} \right)\,\dd \mu = -\int \log \left(w+ (1-w)\frac{\nu}{\mu}\right)\,\dd \mu\\
        &=-\log(w) - \int \log\left(1+ \left(\frac{1-w}{w}\right)\frac{\nu}{\mu}\right)\,\dd \mu
        \leq -\log(w).
    \end{align*}
Similarly, we can derive that $ \kl{\nu}{\pi} \leq - \log (1-w)$.
\end{proof}

\begin{proof}[Proof of \cref{coro:Gaussian-m}]\label{proof-coro-Gaussian-m}
Let $P_0=\cN(-m,I_d)$ and $P_1=\cN(m,I_d)$. By exchanging the order and the orientation of the elements of the canonical basis, we can assume that $j=1$, $k=2$, $m_1>0$ and $m_2>0$. %
 Since  $P_0,P_1\in \cP(\R)^{\otimes d}$, we have
\[
P_0(H_1^{-} \cap H_2^{-}) = P_0( H_1^{-} )  P_0( H_2^{-} )= \Phi(m_1)\Phi(m_2)
\]
and
\[
P_1(H_1^{+} \cap H_2^{+}) = P_1( H_1^{+} )  P_1( H_2^{+} ) =\Phi(m_1)\Phi(m_2) ,
\]
which in turn  implies that  $P_0$ and $P_1$ are $\vae_m$-separated by $H_1^{-} = \{x\in \R^d:x_1<0\}$ and $H_2^{-} = \{x\in \R^d:x_2<0\}$ with $\vae_m = 1- \Phi(m_1)\Phi(m_2)$. %
Under  \cref{thm:MFVI-concentration}, if
\[
1- \Phi(m_1)\Phi(m_2)= \vae_m \leq \min \left\{e^{-2b}, e^{-\frac{b}{2(\sqrt{\delta}-\delta)}}\right\} = e^{-\frac{b}{2(\sqrt{\delta}-\delta)}},
\]
then $t=b/\log(\vae_m^{-1})\leq2(\sqrt{\delta}-\delta)\leq1/2$, which is equivalent to $(1-t-\sqrt{1-2t})/2\leq\delta$. Thus~\eqref{eq:Gaussian-result} holds.
 Furthermore, using $P_0$ and $P_1$ as competitors in the MFVI problem, \cref{lemma:kl-mixture-bound} yields that \[
b\leq \log(2) + \min\{-\log (w), -\log (1-w)\},
\]
so that the second claim follows from the fact that $x\mapsto 1- e^{-\frac{x}{2(\sqrt{\delta}-\delta)}}$ is increasing.

\end{proof}

\section{Proof of Results in \multimodecollapsesectionref}
\begin{proof}[Proof of \Cref{thm:multimodal}]
By \eqref{eq:strong-separation}, we have $0<(K-1)\vae<1$, $b<\infty$, and
$0<\tau_\vae\leq2/(K+2)$. In particular,
$\kl{\mu^*}{\pi}=b-\log(2)<\infty$, so $\mu^*\ll\pi$.

Since $R_r=\bigcap_{s\ne r} (H_{j(r,s)}^{-} \cap H_{k(r,s)}^{-})$ is an intersection of half-spaces, we can write $R_r=\prod_{\ell=1}^{d}I_{r\ell}$ where $I_{r\ell}$ are intervals, possibly unbounded. Moreover, the union bound gives
\begin{equation}\label{eq:P_R}
   P_r(R_r^c)
 \leq\sum_{s\ne r}
 P_r\!\left((H_{j(r,s)}^-\cap H_{k(r,s)}^-)^c\right)
 \leq(K-1)\vae<1.
\end{equation}
Thus every $R_r$ is nonempty. For $r\ne s$, the intervals $I_{r,j(r,s)}$ and $I_{s,j(r,s)}$
are disjoint, and likewise for coordinate $k(r,s)$.
Hence, $R_r\cap R_s=\varnothing$. Therefore, by $R_r\subseteq S=\bigcup_{r=0}^{K-1}R_r$ and by~\eqref{eq:P_R},
\begin{equation}\label{eq:pi-Sc}
   \pi(S^c)=\sum_{r=0}^{K-1}w_rP_r(S^c)
 \leq\sum_{r=0}^{K-1}w_rP_r(R_r^c)\leq(K-1)\vae.
\end{equation}

Choose $r_*$ that maximizes $r\mapsto \mu^*(R_r)$, and set
$\rho=\sum_{r\ne r_*}\mu^*(R_r)$ and $\delta=\mu^*(S^c)$.
Removing the region of largest mass minimizes the sum of the remaining
masses, so $\rho$ is the left-hand side of \eqref{eq:square-root}. Moreover, it holds that
\begin{equation}\label{eq:Rr-lower}
  \mu^*(R_{r_*})+\rho=1-\delta,\qquad \mu^*(R_{r_*})\geq(1-\delta)/K.
\end{equation}
If $\rho=0$, the conclusion is immediate. Suppose henceforth that
$\rho>0$. 

We first relate $\rho$ to $\delta$: under a product measure,
mass in distinct component regions leads to mass outside $S$. Fix $r\ne s$ and set $j=j(r,s)$, so that
$I_{rj}\cap I_{sj}=\varnothing$. To identify the mass placed outside
$S$, exchange the $j$-th coordinate intervals of $R_r$ and $R_s$ and define
\[
\begin{aligned}
 C_{rs}&=I_{sj}\times \prod_{\ell\neq j} I_{r\ell} := \{x:x_j\in I_{sj},\ x_\ell\in I_{r\ell}
                  \text{ for }\ell\ne j\},\\
 C_{sr}&= I_{rj}\times \prod_{\ell\neq j} I_{s\ell} = \{x:x_j\in I_{rj},\ x_\ell\in I_{s\ell}
                  \text{ for }\ell\ne j\}.
\end{aligned}
\]
We claim that $C_{rs}\subseteq S^c$. First,
$C_{rs}\cap R_r=\varnothing$ because their $j$-th intervals are
$I_{sj}$ and $I_{rj}$. For any $t\ne r$, at least one coordinate
$\ell\in\{j(r,t),k(r,t)\}$ differs from $j$.
In that coordinate, $C_{rs}$ retains the interval $I_{r\ell}$, and
$I_{r\ell}\cap I_{t\ell}=\varnothing$. Hence
$C_{rs}\cap R_t=\varnothing$ for every $t$, proving the claim.
The same argument gives $C_{sr}\subseteq S^c$.
Also $C_{rs}\cap C_{sr}=\varnothing$, since their $j$-th intervals
are disjoint.

Since $\mu^*$ is a product measure, we have $\mu^*(C_{rs})=\mu_j^*(I_{sj})\prod_{\ell\ne j}\mu_\ell^*(I_{r\ell})$,
and similarly for $C_{sr}$. Therefore, by the arithmetic--geometric
mean inequality, we derive that
\begin{equation}\label{eq:multimodal-pairwise-mass}
\begin{aligned}
 \delta\geq\mu^*(C_{rs})+\mu^*(C_{sr})&\geq 2\sqrt{ \mu^*(C_{rs})\mu^*(C_{sr})}= 2\sqrt{\mu^*(R_r)\mu^*(R_s)}.
\end{aligned}
\end{equation}
Taking $s=r_*$ in \eqref{eq:multimodal-pairwise-mass} and squaring gives
$\mu^*(R_{r_*})\mu^*(R_r)\leq\delta^2/4$ for each $r\ne r_*$. Summing these $K-1$ bounds yields
\begin{equation}\label{eq:multimodal-mass-constraint}
 \rho(1-\delta-\rho)=\rho \mu^*(R_{r_*})
 =\sum_{r\ne r_*}\mu^*(R_{r_*})\mu^*(R_r)\leq\frac{K-1}{4}\delta^2.
\end{equation}
Since $\rho>0$, we have $\mu^*(R_{r_*})>0$, so \eqref{eq:multimodal-mass-constraint}
implies $\delta>0$. Also $\delta<1$ because
$\mu^*(R_{r_*})+\rho=1-\delta>0$.

We now derive an upper bound for $\delta$ using the argument similar to the one used in the proof of~\cref{thm:MFVI-concentration}. Since $\mu^*\ll\pi$ and
$\delta>0$, we have $0<\pi(S^c)\leq(K-1)\vae<1$.
Applying \cref{lemma:KL-lower-bound} separately to $S$ and $S^c$ gives
\[
\begin{aligned}
 \kl{\mu^*}{\pi}
 &=\int_S\log\left(\frac{\dd\mu^*}{\dd\pi}\right)\,\dd\mu^*
   +\int_{S^c}\log\left(\frac{\dd\mu^*}{\dd\pi}\right)\,\dd\mu^*\\
 &\geq(1-\delta)\log\frac{1-\delta}{\pi(S)}
       +\delta\log\frac{\delta}{\pi(S^c)}.
\end{aligned}
\]
Since $\pi(S^c)\leq(K-1)\vae$ and $-(1-\delta)\log\pi(S)>0$, it follows that
\[
\begin{aligned}
 \kl{\mu^*}{\pi}
 &>(1-\delta)\log(1-\delta)+\delta\log\delta
       +\delta\log(((K-1)\vae)^{-1})\\
 &\geq-\log(2)+\delta\log(((K-1)\vae)^{-1}),
\end{aligned}
\]
where the last inequality uses
$(1-\delta)\log(1-\delta)+\delta\log\delta\geq-\log(2)$.
Therefore,
\begin{equation}\label{eq:multimodal-outside-mass}
 \delta<\frac{ \kl{\mu^*}{\pi} + \log(2)}{\log(((K-1)\vae)^{-1})}=\frac{b}{\log(((K-1)\vae)^{-1})}=\tau_\vae\leq\frac{2}{K+2}.
\end{equation}

We now combine \eqref{eq:multimodal-mass-constraint} and
\eqref{eq:multimodal-outside-mass} to bound $\rho$.
By \eqref{eq:multimodal-mass-constraint}, we have
$\rho(1-\delta-\rho)\leq(K-1)\delta^2/4$. Thus, 
\[
 \rho \leq \frac{1-\delta-\sqrt{(1-\delta)^2-(K-1)\delta^2}}{2}\quad {\rm or}\quad \rho \geq \frac{1-\delta + \sqrt{(1-\delta)^2-(K-1)\delta^2}}{2} 
\]
Furthermore, \eqref{eq:multimodal-outside-mass} gives
$\delta<2/(K+2)$, hence $\delta<2(1-\delta)/K$ and
$\sqrt{(1-\delta)^2-(K-1)\delta^2}\allowbreak>(K-2)(1-\delta)/K$. Therefore, 
\[
\rho \leq \frac{1-\delta-\sqrt{(1-\delta)^2-(K-1)\delta^2}}{2}
\]
since otherwise, 
\[
 \mu^*(R_{r_*})=1-\delta-\rho
 \leq\frac{1-\delta-\sqrt{(1-\delta)^2-(K-1)\delta^2}}2
 <\frac{1-\delta}{K},
\]
contradicting $\mu^*(R_{r_*})\geq(1-\delta)/K$ in~\eqref{eq:Rr-lower}. Hence,
\[
 \rho\leq F_K(\delta),\qquad \text{where}\,\,
 F_K(t):=\frac{1-t-\sqrt{(1-t)^2-(K-1)t^2}}2.
\]
Note that $F_K$ is continuous and increasing on $[0,2/(K+2)]$. Together with \eqref{eq:multimodal-outside-mass}, we obtain
\[
 \min_{0\leq r\leq K-1}\sum_{s\ne r}\mu^*(R_s)
 =\rho\leq F_K(\delta)\leq F_K(\tau_\vae),
\]
which proves \eqref{eq:square-root}.
\end{proof}

\section{Proof and Details of Results in \rovisectionref}\label{sec:partial-gradient}
Fix $O\in \cO(d)$, we compute the partial gradients of $(\lambda, v) \mapsto \kl{\mu_{\lambda, v}}{\pi_O}$. Recall that
\begin{align*}
 \kl{\mu_{\lambda, v}}{\pi_O}&= \int V\circ O \circ T_\theta \dd \rho + H\bigl(\mu_{\lambda,v}\bigr)  + \log (Z)\\
   & =  \int V\circ O\left(\sum_{S \in \cM} \lambda_S S(x) + v\right) \rho(\rd x)\\
   &\quad - \int \log \left(\det \left( \sum_{S \in \cM} \lambda_S D S(x) \right)\right) \rho(\rd x)  + \textsc{const}.
\end{align*}
Recall also that $T_{\lambda,v}(x)= \sum_{S \in \cM} \lambda_S S(x)+ v$, then the partial gradients of $\kl{\mu_{\lambda, v}}{\pi_O}$ w.r.t.$(\lambda, v)$ are computed as:
\begin{align*}
    \partial_v  \kl{\mu_{\lambda, v}}{\pi_O}  = \int O^\top \nabla V(OT_{\lambda,v}(x)) \rho(\rd x), 
\end{align*}
and for any $T\in \cM$, then 
\begin{equation*}
     \partial_{\lambda_T} \kl{\mu_{\lambda, v}}{\pi_O} = \int \bigg[\left\langle O T(x),  \nabla V(OT_{\lambda,v}(x))\right\rangle -\tr \left([DT_{\lambda, v}(x)]^{-1} D T(x) \right)  \bigg]     \rho(\rd x). 
\end{equation*}

\begin{lemma}\label{lemma:kl-mixture}
    Let $\mu = \sum_{i=1}^n w_i \mu_i$ and $\nu = \sum_{i=1}^n \eta_i \nu_i$ where $\mu_i,\nu_i\in \cP(\R^d)$, and $\sum_{i=1}^n w_i = 1$, $\sum_{i=1}^n \eta_i = 1$, $w_i,\eta_i \geq 0$, then
    \[
    \kl{\mu}{\nu} \leq \kl{w}{\eta} + \sum_{i=1}^{n} w_i \kl{\mu_i}{\nu_i}.
    \]
\end{lemma}
\begin{proof}
Define the measure $\mu_{XZ}\in\cP(\R^d\times[n])$ as
\begin{equation*}
    \mu_{XZ}(A\times\{z\})= \mu_Z(\{z\})\,\mu_{X\mid Z}(A,z)
 := w_{z}\,\mu_z(A),
\qquad \forall\,A\in\cB(\R^d),\ z\in[n].
\end{equation*}
Define $\nu_{XZ}$ similarly. We note that $\mu$ is the first marginal of $\mu_{XZ}$. By the chain rule for the KL divergence \cite[Theorem~2.5.3]{Cover2006} and the data processing inequality \cite[Lemma~1.6]{nutz2021introduction}, we conclude that
\begin{align*}
\kl{\mu}{\nu}
\leq \kl{\mu_{XZ}}{\nu_{XZ}}
&= \kl{\mu_Z}{\nu_Z}
  + \int \kl{\mu_{X\mid Z}(\cdot,z)}{\nu_{X\mid Z}(\cdot,z)}\,\mu_Z(\rd z) \\
&= \kl{w}{\eta} + \sum_{i=1}^n w_i\,\kl{\mu_i}{\nu_i}.
\end{align*}

\end{proof}

The following result follows from a direct computation of the KL divergence between two Gaussians.
\begin{lemma}\label{lemma:kl-Gaussian}
    Consider $\mu_0 \sim \cN(m_0, \Sigma_0)$, $\mu_1 \sim \cN(m_1, \Sigma_1)$ with $m_0, m_1\in \R^d$, $\Sigma_0,\Sigma_1 \in \R^{d\times d}$. Then,
\begin{equation*}
\kl{\mu_0}{\mu_1}  = \dfrac{1}{2} \log\left( \dfrac{\det \Sigma_1}{\det \Sigma_0} \right)+ \dfrac{1}{2}({\rm tr}(\Sigma_1^{-1}\Sigma_0) - d) +\dfrac{1}{2}(m_0 - m_1)^\top \Sigma_1^{-1} (m_0 -m_1)\,.
\end{equation*}
\end{lemma}
Now, we are ready to prove \cref{prop:rovi-kl-gaussian-mixture}.
\begin{proof}[Proof of \cref{prop:rovi-kl-gaussian-mixture}]\label{proof:prop:rovi-kl-gaussian-mixture}
Suppose first that $m_0\neq m_1$, and let $\Delta := \|m_1-m_0\|_2$.
Choose $O\in\cO(d)$ such that $O^\top(m_1-m_0)=\Delta e_1$, where $e_1=(1,0,\ldots,0)^\top$. Translations preserve product measures, and KL divergence is invariant under a common translation. Thus we may assume that $(m_0+m_1)/2=0$. With $\pi_O=(O^\top)_\sharp\pi$ and $\Sigma_O:=O^\top\Sigma O$, we obtain
\[
\pi_O=w\cN(-\tfrac{\Delta}{2}e_1,\Sigma_O)
 +(1-w)\cN(\tfrac{\Delta}{2}e_1,\Sigma_O).
\]
Let $D=({\rm diag}(\Sigma_O^{-1}))^{-1}=({\rm diag}(O^\top\Sigma^{-1}O))^{-1}$. We define $\mu_i=\cN(0,D_{ii})$ for $i\geq2$ and $\mu_1=w\cN(-\Delta/2,D_{11})+(1-w)\cN(\Delta/2,D_{11})$. Then
\[
\mu=\bigotimes_{i=1}^d\mu_i
 =w\cN(-\tfrac{\Delta}{2}e_1,D)+(1-w)\cN(\tfrac{\Delta}{2}e_1,D).
\]
By \cref{lemma:kl-mixture}, \cref{lemma:kl-Gaussian} and $\tr(\Sigma_O^{-1}D)=d$, we have
\begin{align*}
\kl{O_\sharp\mu}{\pi}=\kl{\mu}{\pi_O}
 &\leq\frac12\log\left(\frac{\det\Sigma_O}{\det D}\right)
       +\frac12\bigl(\tr(\Sigma_O^{-1}D)-d\bigr)\\
 &=\frac12\log\det\Sigma
       +\frac12\sum_{i=1}^d\log\bigl((\Sigma_O^{-1})_{ii}\bigr).
\end{align*}
Let $v_1 := \frac{m_1 - m_0}{\|m_1 - m_0 \|_2}$.  Consider the matrix
\begin{equation*}
M := \left(I - v_1 v_1^\top \right) \Sigma^{-1} \left(I - v_1 v_1^\top \right). 
\end{equation*}
The matrix $M$ vanishes on $v_1$ and is positive definite on $v_1^\perp$, so it has rank $d-1$. Let $v_2,\ldots,v_d$ be orthonormal eigenvectors of $M$ spanning $v_1^\perp$. Choose $O$ such that $Oe_i=v_i$ for all $i\in[d]$. Then $O^\top(m_1-m_0)=\Delta e_1$, as required, and
\begin{equation}\label{eq:kl-rovi-bound}
\begin{aligned}
\kl{O_\sharp\mu}{\pi}
&=\kl{\mu}{\pi_O}
\leq\frac12\log\det\Sigma
 +\frac12\sum_{i=1}^d\log\bigl((Oe_i)^\top\Sigma^{-1}(Oe_i)\bigr)\\
&=\frac12\left(\log\det\Sigma
 +\sum_{i=1}^d\log\left(v_i^\top\Sigma^{-1}v_i\right)\right).
\end{aligned}
\end{equation}
Taking the infimum over rotated product measures proves \eqref{eq:rovi-kl-gm}.
If $v_1$ is an eigenvector of $\Sigma$, write the eigendecomposition $\Sigma=V\Lambda V^\top$ where $V=[u_1,\ldots,u_d] \in \cO(d)$ with $u_1 = v_1$. Then
\[
M = \bigl(I - v_1 v_1^\top\bigr) V\Lambda^{-1}V^\top \bigl(I - v_1 v_1^\top\bigr)
 =\left[ 0, V_{-1} \right] \Lambda^{-1} \left[0, V_{-1} \right]^\top
 = V_{-1} \Lambda_{-1}^{-1} V_{-1}^\top,
\]
where $V_{-1}=[u_2,\ldots,u_d]$ and $\Lambda_{-1}=\mathrm{diag}(\lambda_2,\ldots,\lambda_d)$. Thus $u_2, \cdots, u_d$ are also eigenvectors of $M$. In particular, we may choose $v_i = u_i$, $i\geq 2$ in \eqref{eq:kl-rovi-bound}, yielding that
\begin{equation*}
    \kl{\mu}{\pi_O} \leq  \frac{1}{2}  \log (\det (\Sigma)) - \sum_{i = 1}^d \frac{1}{2} \log \lambda_i  = 0,
\end{equation*}
where $\lambda_1,\ldots, \lambda_d$ are the eigenvalues of $\Sigma$. Moreover, we may choose $O = V$.

Finally, if $m_0=m_1=m$, then $\pi=\cN(m,\Sigma)$. Choose an orthogonal diagonalizer $O$ of $\Sigma$. The measure $\mu=(O^\top)_\sharp\pi=\cN(O^\top m,O^\top\Sigma O)$ is a product Gaussian and satisfies $O_\sharp\mu=\pi$, proving that the infimum is zero.
\end{proof}

\section{Gaussian Experiment Settings}
\label{sec:gaussian-experiment-settings}
The following six settings correspond to panels (a)--(f) of \Cref{fig:GM-comparison}.
In these synthetic runs, we initialize $O_0=I_d$ and evaluate the transport and rotation gradients at the current iterate before updating the parameters. The rotation update uses the raw QR output, whereas \Cref{algo:RoVI} fixes the signs through a positive diagonal in $R$.

\begin{enumerate}[label = (\roman*)]
  \item When $w = 1$ and $\Sigma_1$ is non-diagonal, MFVI produces an underdispersed Gaussian approximation while RoVI accurately recovers the exact target (\Cref{fig:GM-comparison}(a), the target is $\cN\left([0, 0]^\top,
      \begin{bmatrix}
        1.8 & 1.2 \\[2pt]
        1.2 & 1.0
      \end{bmatrix}\right)$).

  \item When $w \in (0,1)$ and $m_1, m_2$ lie on the same horizontal axis, all three methods (MFVI, RoVI, and LMC) correctly recover the target distribution (\Cref{fig:GM-comparison}(b), the target is $0.4\,\cN\left([-2.5, 0]^\top, I_2\right) + 0.6\,\cN\left([2.5, 0]^\top, I_2\right)$).

  \item When $m_1, m_2$ are misaligned but $\Sigma_1 = \Sigma_2 = I_2$, RoVI successfully captures both modes, whereas MFVI suffers from mode collapse (\Cref{fig:GM-comparison}(c), the target is $0.5\,\cN\left([-2.5, -1.5]^\top, I_2\right) + 0.5\,\cN\left([2.0, 1.0]^\top, I_2\right)$).

  \item When $m_1, m_2$ are far apart with $\Sigma_1 = \Sigma_2 = I_2$, both LMC and RoVI recover the two modes, while MFVI collapses to a single mode (\Cref{fig:GM-comparison}(d), the target is $0.5\,\cN\left([-8, -8]^\top, I_2\right) + 0.5\,\cN\left([8, 8]^\top, I_2\right)$).

  \item When $\Sigma_1 \neq \Sigma_2$, RoVI is robust to the asymmetry in covariances and accurately matches LMC marginals, while MFVI underestimates dispersion (\Cref{fig:GM-comparison}(e)). The target is
  \[
  0.5\,\cN\left([-2.0, -2.0]^\top,
  \begin{bmatrix}
  1.552 & 0.864 \\[2pt]
  0.864 & 1.048
  \end{bmatrix}\right)
  + 0.5\,\cN\left([2.5, 2.0]^\top,
  \begin{bmatrix}
  0.7 & 0 \\[2pt]
  0 & 1.1
  \end{bmatrix}\right).
  \]

  \item When both means are far apart and covariances are asymmetric, all three methods suffer from mode collapse (\Cref{fig:GM-comparison}(f)). The target is
  \[
  0.5\,\cN\left([-5.0, -5.0]^\top,
  \begin{bmatrix}
  2.0 & -0.6 \\[2pt]
  0.6 & 2.0
  \end{bmatrix}\right)
  + 0.5\,\cN\left([2.5, 2.0]^\top,
  \begin{bmatrix}
  0.7 & 0 \\[2pt]
  0 & 1.1
  \end{bmatrix}\right).
  \]
\end{enumerate}

\begingroup
\section{Posterior Experiment Settings}
\label{sec:posterior-experiment-settings}
For the experiments in \Cref{sec:posterior-mode-retention}, we represent the posterior in unconstrained coordinates, using log scales and logit weights with their density Jacobians. We fix a diagonal scaling matrix $D$ whose entries are the square roots of the average within-mode Laplace variances. The variational distribution is obtained from $z\sim\rho=\cN(0,I_d)$ through
\begin{equation*}
    F_{\theta,O}(z)=DO T_\theta(z),
    \qquad T_{\theta,j}(z_j)=v_j+e^{a_j}s_j(z_j).
\end{equation*}
Here $\theta$ collects the spline parameters, shifts $v_j$, and log scales $a_j$. For RoVI, we jointly optimize $\theta$ and $O=O_0(I-C)^{-1}(I+C)$, where $C^\top=-C$ and $O_0$ is the initial orthogonal matrix. For MFVI, we fix $O=I_d$. In the coordinates rescaled by $D$, this construction and \Cref{algo:RoVI} optimize the same reverse-KL objective with different transport-map parameterizations.

For each coordinate, we use a monotone rational-quadratic spline with 32 bins on $[-8,8]$, following \href{https://papers.neurips.cc/paper_files/paper/2019/hash/7ac71d433f282034e088473244df8c02-Abstract.html}{Durkan et al. (2019)}. To describe one spline, let $(x_k,y_k)$ be its knots and $r_k$ its knot derivatives. On the interval $[x_k,x_{k+1}]$, the spline has the form
\begin{equation*}
    s(t)=y_k+h_k\,
    \frac{\delta_k\xi^2+r_k\xi(1-\xi)}
    {\delta_k+(r_k+r_{k+1}-2\delta_k)\xi(1-\xi)},
    \qquad \xi=\frac{t-x_k}{w_k},
\end{equation*}
where $w_k=x_{k+1}-x_k$, $h_k=y_{k+1}-y_k$, and $\delta_k=h_k/w_k$. These rational-quadratic pieces are determined by the learned bin widths, heights, and knot derivatives, which control the shape of each marginal transport. We constrain widths and heights to sum to 16, with minimum values $10^{-8}$, and interior derivatives to exceed $10^{-4}$. The endpoint knots are $(-8,-8)$ and $(8,8)$ with unit derivatives, so the spline joins continuously differentiable identity tails.

We fit and select the variational distributions by maximizing the evidence lower bound (ELBO). Let $\widetilde\pi$ be the unnormalized posterior density in the unconstrained coordinates and $Z$ its normalizing constant. For the density $q$ induced by $F_{\theta,O}$, the objective and its Monte Carlo estimate are
\begin{equation*}
\begin{aligned}
    \mathcal L(q)
    &=\mathbb E_q[\log\widetilde\pi(X)-\log q(X)]
      =\log Z-\kl{q}{\pi},\\
    \widehat{\mathcal L}_M(q)
    &=\frac1M\sum_{r=1}^M
      \bigl[\log\widetilde\pi(F_{\theta,O}(z^{(r)}))
      -\log q(F_{\theta,O}(z^{(r)}))\bigr],
      \qquad z^{(r)}\overset{\mathrm{iid}}\sim\rho.
\end{aligned}
\end{equation*}
Here the density is evaluated by the change-of-variables identity
\begin{equation*}
    \log q(F_{\theta,O}(z))
    =\log\rho(z)-\log\det D
      -\sum_{j=1}^d\bigl[a_j+\log s_j'(z_j)\bigr].
\end{equation*}
Note that the orthogonal matrix contributes zero to the log absolute Jacobian determinant, since $|\det O|=1$.

For each of five paired runs, each method uses local and equal-weight two-mode Laplace initializations, with the initial RoVI axis along the mode contrast. Each candidate receives 10,000 Adam updates using 256 draws per update, a learning rate cosine-decaying from $0.001$ to $0.0001$, and gradient-norm clipping at 100. We select the candidate with the highest $\widehat{\mathcal L}_{1000}$ using fresh validation draws, then estimate its mode-coverage measures from 16,384 independent evaluation draws. The methods share base random draws within each run. \Cref{tab:posterior-mode-retention} reports means and standard deviations across the five selected fits per method, and \Cref{fig:posterior-mode-retention,fig:posterior-separated-mixture} display the first run.

\begin{figure}[ht]
\centering
\includegraphics[width=\linewidth]{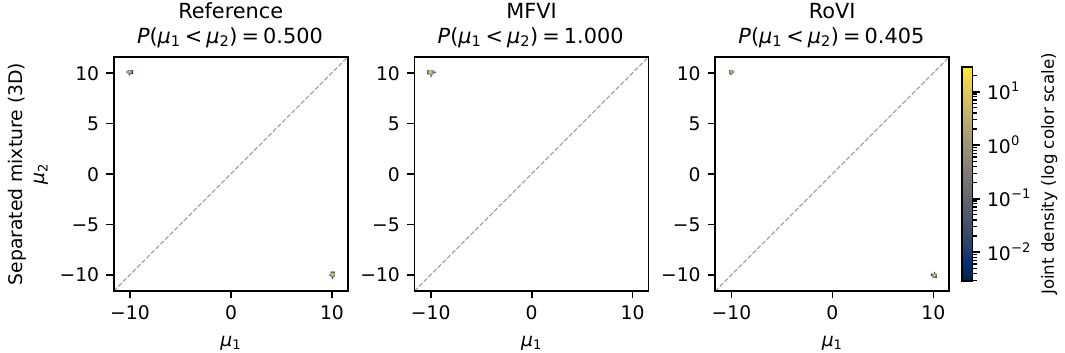}
\caption{Joint densities for the separated-mixture posterior (3D). The columns show the reference and the fitted MFVI and RoVI approximations, with the probability $P(\mu_1 < \mu_2)$ above each panel. RoVI retains both modes, while MFVI concentrates on one.}
\label{fig:posterior-separated-mixture}
\end{figure}

\endgroup
\end{document}